\documentclass{article}
\usepackage{times}
\usepackage{fullpage}
\usepackage[numbers]{natbib}

\RequirePackage[dvipsnames,x11names]{xcolor}
\usepackage[utf8]{inputenc} % allow utf-8 input
\usepackage[T1]{fontenc}    % use 8-bit T1 fonts
\RequirePackage[colorlinks=true,allcolors=blue!80!red]{hyperref}
\usepackage{nicefrac}       % compact symbols for 1/2, etc.
\usepackage{microtype}

\usepackage{algorithm}
\usepackage{algorithmic}
\RequirePackage{adjustbox}
\usepackage{newfloat}
\RequirePackage[disable,textsize=scriptsize,color=yellow!25,linecolor=brown]{todonotes}
\usepackage{booktabs}       % professional-quality tables
\RequirePackage{amsmath,amssymb,amsfonts,bm}
\usepackage{subfig}
\usepackage{wrapfig,booktabs}
\RequirePackage[skip=.2ex]{caption}

\usepackage{Defs}
\newcommand{\ehdr}[1]{\noindent\emph{#1.}}

\newcommand{\our}{\textsc{Flex\-Sub\-Net}}
\newcommand{\xhdr}[1]{\noindent\textbf{#1.}}
\makeatletter
\newlength{\eqnsep}
\g@addto@macro{\normalsize}{%
\setlength{\abovedisplayskip}{\eqnsep}%
\setlength{\abovedisplayshortskip}{\eqnsep}%
\setlength{\belowdisplayskip}{\eqnsep}%
\setlength{\belowdisplayshortskip}{\eqnsep}}
\let\c@table\c@figure
\makeatother
\dbltextfloatsep \floatsep
\textfloatsep \floatsep
\intextsep \floatsep

\DeclareMathOperator*{\argmax}{argmax}

\def\d{\mathrm{d}}
\renewcommand{\sec}{\underline}

\renewcommand{\cite}{\citep}

\newcommand{\ttc}{\diffc_{\theta}}
\newcommand{\ga}{\kappa(\alpha)}
\newcommand{\galphat}{g_{\theta}}

\newcommand{\best}{\textbf}
\newcommand{\secbest}{\underline}
\newcommand{\reg}{\rho}
 
\newcommand{\deepset}{Deep-set}
\newcommand{\settx}{Set-transformer}

\def\ztitle{Neural Estimation of Submodular Functions\\ with Applications to  Differentiable Subset Selection}
\date{}
\title{\ztitle}

\author{Abir De, Soumen Chakrabarti \\
Indian Institute of Technology Bombay\\
\texttt{\{abir,soumen\}@cse.iitb.ac.in}
}

\begin{document}

\maketitle
\begin{abstract}
Submodular functions and variants, through their ability to characterize diversity and coverage, have emerged as a key tool for data selection and summarization.  Many recent approaches to learn submodular functions suffer from limited expressiveness. In this work, we propose \our, a family of flexible neural models for both monotone and non-monotone submodular functions. To fit a latent submodular function from (\tset, \tvalue) observations, \our\ applies a concave function on modular functions in a recursive manner. We do not draw the concave function from a restricted family, but rather learn from data using a highly expressive neural network that implements a differentiable quadrature procedure. Such an expressive neural model for concave functions may be of independent interest.  Next, we extend this setup to provide a novel characterization of monotone $\alpha$-submodular functions, a recently introduced notion of approximate submodular functions.  We then use this characterization to design a novel neural model for such functions. Finally, we consider learning submodular set functions under distant supervision in the form of (\tuniv, \thigh) pairs.  This yields a novel subset selection method based on an order-invariant, yet greedy sampler built around the above neural set functions. Our experiments on synthetic and real data show that \our\ outperforms several baselines.
\end{abstract}

%\vspace{-1mm}
\section{Introduction}
%\vspace{-1mm}
\label{sec:Intro}

Owing to their strong characterization of diversity and coverage, submodular functions and their extensions, \viz, weak and approximate submodular functions, have emerged as a powerful machinery in data selection tasks~\cite{killamsetty2020glister, wei2015submodularity, hashemi2019submodular, ren2018learning, zhang2018generalized, powers2018differentiable,b1,b2,b3}.
% \cite{killamsetty2020glister,wei2015submodularity,hashemi2019submodular,ren2018learning,zhang2018generalized,powers2018differentiable}.
% \cite{el2020optimal,khanna2017scalable,das2011submodular}, are set functions which satisfy the law of diminishing returns. 
%  greedy optimization routines~\cite{nemhauser1978analysis,mirzasoleiman2014lazier,barbosa2016new,Qian2017a,Horel2016},
% \cite{killamsetty2020glister,wei2015submodularity,hashemi2019submodular,ren2018learning,zhang2018generalized,powers2018differentiable}.
%
We propose trainable parameterized families of submodular functions under two supervision regimes. In the first setting, the goal is to estimate the submodular function based on (\tset, \tvalue) pairs, where the function outputs the value of an input \tset.  This problem is hard in the worst case for $\text{poly}(n)$ value oracle queries~\cite{reb}.  This has  applications in auction design where one may like to learn a player's valuation function based on her bids~\cite{balcan2011learning}.  In the second setting, the task is to learn the submodular function under the supervision of (\tuniv, \thigh) pairs, where \thigh\ potentially maximizes the underlying function against all other subsets of \tuniv.  The trained function is expected to extract high-value subsets from freshly-specified perimeter sets.
This scenario has applications in  itemset prediction in recommendation~\cite{tschiatschek2016learning, tschiatschek2018differentiable}, data summarization \cite{kothawade2020deep, badanidiyuru2014streaming, bairi2015summarization, dasgupta2013summarization, Singla2016}, \etc.

\subsection{Our contributions}

Driven by the above motivations, we propose (i)~a novel family of highly expressive neural models
for submodular and \asb\ functions which can be estimated
under the supervisions of both (\tset, \tvalue)   
and (\tuniv, \thigh) pairs; (ii)~a novel permutation adversarial training method for differentiable subset selection, which efficiently trains  submodular and \asb\ functions based on (\tuniv, \thigh) pairs.  We provide more details below.

\xhdr{Neural models for submodular functions} 
% \todo{Or FlexSubNet?}
We design \our, a family of flexible neural  models for monotone, non-monotone submodular functions and monotone \asb\ functions.
% We call it  \our.
%

\ehdr{--- Monotone submodular functions}  We model a monotone submodular function using a recursive neural model which outputs a concave composed submodular function~\cite{csm2,csm3,lin2011class} at every step of recursion. Specifically, it first computes a linear combination of the submodular function computed in the previous step and a modular function and, then applies a concave function to the result to output the next submodular function.

\ehdr{--- Monotone \asb\ function} Our proposed model for submodular function rests on the principle that a concave composed submodular function is always submodular. However, to the best of our knowledge, there is no known result for an $\alpha$-submodular function. We address this gap by showing that  an $\alpha$-submodular function can be represented by applying a mapping $\diffc$ on a positive modular set function, where $\diffc$ satisfies a second-order differential inequality. Subsequently, we provide a neural model representing the universal approximator of $\diffc$, which in turn is used  for modeling an $\alpha$-submodular function.

\ehdr{--- Non-monotone submodular functions}
By applying a non-monotone concave function on modular function,
we extend our model to non-monotone submodular functions.

\iffalse \begingroup \color{Tomato2}
Unlike early works~\cite{tschiatschek2014learning,tschiatschek2018differentiable,lin2012learning}, we do not assume any simple, fixed parametric form in our models. Moreover, in contrast to \citet{bilmes2017deep}, who propose deep submodular networks with a fixed concave function, we do not fix the concave function a-priori in our model, but rather design a trainable model for it. We ensure concavity by enforcing negativeness on the second derivative of the underlying function, then using a differentiable quadrature method to estimate it.  Thus, we provide a universal approximators for univariate concave functions, which may be of independent interest.
\endgroup \fi

Several recent models learn subclasses of submodular functions~\citep{lin2012learning, tschiatschek2014learning, tschiatschek2018differentiable}.  \citet{bilmes2017deep} present a thorough theoretical characterization of the benefits of network depth with fixed concave functions, in a general framework called deep submodular functions (DSF).
DSF leaves open all design choices: the number of layers, their widths, the DAG topology, and the choice of concave functions. 
All-to-all attention has replaced domain-driven topology design in much of NLP \cite{devlin2018bert}.  Set transformers \citep{lee2019set} would therefore be a natural alternative to compare against DSF, but need more memory and computation.
Here we explore the following third, somewhat extreme trade-off:
we restrict the topology to a single recursive chain,
thus providing an effectively plug-and-play model with no topology and minimal 
hyperparameter choices (mainly the length of the chain).
However, we compensate with a more expressive, trainable concave function that is shared across all nodes of the chain.
Our experiments show that our strategy improves ease of training and predictive accuracy beyond both set transformers and various DSF instantiations with fixed concave functions.
% Related work is discussed in more detail in Appendix\,\ref{app:limitation}.

\xhdr{Permutation insensitive differentiable subset selection} 
It is common \cite{tschiatschek2014learning, sakaue2021differentiable,powers2018differentiable} to select a subset from a given dataset by sequentially sampling   elements using a softmax distribution obtained from the outputs of a set function on various sets of elements. At the time of learning
set functions based on (\tuniv, \thigh) pairs,
this protocol naturally results in order-sensitivity in the training process for learning set functions.
To mitigate this, we propose a novel max-min optimization.
Such a formulation sets forth a game between an adversarial permutation generator and the set function learner --- where
the former generates the worst-case permutations to induce minimum
likelihood of training subsets and the latter keeps maximizing the likelihood function 
until the estimated parameters become permutation-agnostic.
To this end, we use a Gumbel-Sinkhorn neural network \cite{mena2018learning,permgnn,demaximum,roy2022interpretable} as a neural surrogate of hard permutations,  that expedites the underlying training process and allows us to avoid combinatorial search on large permutation spaces. 

% We summarize a large dataset via sequential sampling from a softmax distribution induced by the marginal gains of the submodular (or $\alpha$-submodular)
%  set function.
%  At each iteration, our summarization model first computes marginal gains of the set function on each element from the current set of candidates,
%  feeds them as logits into a softmax distribution and then samples the elements. 

\xhdr{Experiments}
We first experiment with several submodular set functions and synthetically generated examples, which show that
\our\ recovers the function more accurately than several baselines.  Later experiments with several real datasets on product recommendation reveal that \our\ can predict the purchased items more effectively and efficiently than several  baselines.
\section{Related work}
%\vspace{-2mm}

\xhdr{Deep set functions} 
Recent years have witnessed a surge of interest in deep learning of set functions.  \citet{zaheer2017deep} showed that any set function can be modeled using a symmteric aggregator on the feature vectors associated with the underlying set members. \citet{lee2019set} proposed a transformer based neural architecture to model set functions. However, their work do not focus on modeling or learning submodular functions in particular. Deep set functions enforce permutation invariance by using symmetric aggregators~\cite{zaheer2017deep, ravanbakhsh2016deep, qi2017pointnet}, which have several applications, \eg, character counting~\cite{lee2019set}, set anomaly detection~\cite{lee2019set}, graph embedding design~\cite{hamilton2017inductive, kipf2016semi,verma2022varscene,samanta2020nevae}, \etc. However, they often suffer from limited expressiveness as shown by \citet{pmlr-v97-wagstaff19a}.  Some work aims to overcome this limitations by sequence encoder followed by learning a permutation invariant network structure~\cite{prateek, permgnn}.  However, none of them learns an underlying submodular model in the context of subset selection.

\xhdr{Learning functions with shape constraints}
Our double-quadrature strategy for concavity is inspired by a series of recent efforts to fit functions with \emph{shape constraints} to suit various learning tasks.
\citet{umnn} proposed universal monotone neural networks (UMNN) was a significant early step in learning univariate monotone functions by numerical integration of a non-negative integrand returned by a `universal' network --- this paved the path for universal monotone function modeling.
\citet{Gupta2020MultiShape} extended to multidimensional shape constraints for supervised learning tasks, for situations were features complemented or dominated others, or a learnt function $y=f(\bm{x})$ should be unimodal.  Such constraints could be expressed as linear inequalities, and therefore possible to optimize using projected stochastic gradient descent.
\citet{Gupta2021PenDerShape} widened the scope further to situations where more general constraints had to be enforced on gradients.
In the context of density estimation and variational inference, a popular technique is to transform between simple and complex distributions via invertible and differentiable mappings using \emph{normalizing flows} \citep{Kobyzev2021NormalizingFlows}, where coupling functions can be implemented as monotone networks.
\todo{@ad this line may be useful in intro}
Our work provides a bridge between shape constraints, universal concavity and differentiable subset selection.

\xhdr{Deep submodular functions (DSF)} 
Early work predominantly modeled a trainable submodular function as a mixture of fixed submodular functions~\cite{lin2012learning,tschiatschek2014learning}.
% Later work learn the mixture weights, keeping the mixture components fixed~\cite{lin2012learning,tschiatschek2014learning}.
If training instances do not fit their `basis' of hand-picked submodular functions, limited expressiveness results.
In the quest for `universal' submodular function networks, \citet{bilmes2017deep} and \citet{bai2018submodular} undertook a thorough theoretical inquiry into the effect of network structure on expressiveness.  Specifically, they modeled submodular functions as an aggregate of concave functions of modular functions, computed in a topological order along a directed acyclic graph (DAG), driven by the fact that a concave function of a monotone submodular function is a submodular function~\cite{csm1, csm2, csm3}.
% Among the key contributions is a proof that a shallow DSF cannot model functions that can be modeled by a deeper one.  
But DSF provides no practical prescription for  picking the concave functions. Each application of DSF will need an extensive search over these design spaces.

\xhdr{Subset selection} 
Subset selection especially under submodular or approximate submodular profit enjoys an efficient greedy maximization routine which admits an approximation guarantee.  Consequently, a wide variety of set function optimization tasks focus on representing the underlying objective as an instance of a submodular function.  At large, subset selection  has a myriad of applications in machine learning, e.g., data summarization~\cite{bairi2015summarization}, feature selection~\cite{khanna2017scalable}, influence maximization in social networks~\cite{kempe2003maximizing,de2018shaping,de2019learning,zarezade2017cheshire}, opinion dynamics~\cite{de2016learning,de2018demarcating,de2019learning,zhang2021learning,koley2021demarcating}, 
efficient learning~\cite{durga2021training,killamsetty2021grad}, human assisted learning~\cite{cuha,ruha,okati2021differentiable}, etc. However, these works do not aim to learn the underlying submodular function from training subsets. 

\xhdr{Set prediction} 
Our work is also related to set prediction.  \citet{zhang2019deep} use a encoder-decoder architecture for set prediction.  \citet{rezatofighi2020learn} provide a deep probabilistic model for set prediction. However, they aim to predict an output  set rather than the set function. 
% Moreover, in contrast to automatic data summarization, the underlying input set
% in these works  does not represent the summary of a larger set.

\xhdr{Differentiable subset selection} 
Existing trainable subset selection methods~\cite{tschiatschek2014learning, kothawade2020deep} often adopt a max-margin optimization approach.
However, it requires solving one submodular optimization problem at each epoch, which renders it computationally expensive. On the other hand, \citet{tschiatschek2018differentiable} provide a probabilistic soft-greedy model which can generate and  be trained on a permutation of subset elements.
But then, the trained model becomes sensitive to this specific permutations.~\citet{tschiatschek2018differentiable} overcome this challenge by presenting
several permutations to the learner, which can be inefficient. 

 \xhdr{One sided smoothness} 
We would like to highlight that our characterization for $\alpha$-submodular function is a special case of one-sided smoothness (OSS) proposed in~\cite{oss1,oss2}. However, the significance of these characterizations are different between their and our work. First, they consider $\gamma$-meta submodular function which is a  different generalisation of submodular functions compared to $\alpha$-submodular functions. Second, the OSS characterization they provide is for the multilinear extension of $\gamma$-meta submodular function, whereas we provide the characterization of $\alpha$-submodular functions itself, which allows direct construction of our neural models. 

 \xhdr{Sample complexity in the context of learning submodular functions}
~\citet{reb} provided an algorithm which outputs a function $\hat{f}(S)$ that approximates an arbitrary monotone submodular function $f(S)$ within a factor $O(\sqrt{n}\log n)$ using poly$(n)$ queries on~$f$.  Their algorithm considers a powerful active probe setting where $f$ can be queried with arbitrary sets. In contrast, \citet{balcan2011learning} consider a more realistic passive setup used in a supervised learning scenario, and designed an algorithm which obtains an approximation of $f(S)$ within factor $O(\sqrt{n})$.

\section{Design of \our}
\label{sec:subm}

In this section, we  first present our notations and then propose a family of flexible neural network models 
for monotone and non-monotone submodular functions and   $\alpha$-submodular functions.
% First we provide nonlinear recursive characterizations of these functions 
% using Proposition~\ref{prop-basic} and then provide their neural parameterization
% to complete the design of our framework, \our.

\subsection{Notation and preliminary results}  

We denote $V=\set{1,2,..,|V|}$ as the ground set or universal set of elements and
$S,T\subseteq V$ as subsets of $V$.  
Each element $s \in V$ may be associated with a feature vector $\feat_{\els}$.
Given a set function  $F:2^V \to \RR$,  
we define the marginal utility $F(\els\given S) :=  F(S\cup \set{\els})-F(S)$. 
The function $F$ is called \emph{monotone}  if $F(\els\given S) \ge 0$ whenever $ S\subset V$ and $\els\in V\cp S$;  $F$ is called \emph{$\alpha$-submodular} with $\alpha >0$ if $F(\els\given S) \ge \alpha F(\els \given T)$ whenever $S \subseteq  T$ and $\els \in V\cp T$~\cite{hashemi2019submodular,zhang2016submodular,el2020optimal}. As a special case, $F$ is submodular if $\alpha=1$ and $F$ is modular if $F(s\given S) = F(s \given T)$. Here, $F(\cdot)$ is called \emph{normalized} if $F(\varnothing)=0$. Similarly, a real function $f:\RR\to\RR$ is \emph{normalized} if $f(0)=0$. 
Unless otherwise stated, we only consider normalized set functions in this paper.
We quote a key result often used for neural modeling for submodular functions~\cite{csm2,csm3,lin2011class}.
%
%
% \todo{maybe hint why this result is so useful}
\begin{proposition}\label{prop-basic}
Given the set function  $F:2^{V} \to \RR^+$ and a real valued  function  $\phi:\RR\to\RR$, 
(i)~the set function $\phi(F(\cdot))$ is monotone submodular if $F$ is monotone submodular and $\phi$ is an increasing concave function; and,
(ii) $\phi(F(\cdot))$ is non-monotone submodular if $F$ is positive modular and $\phi$ is non-monotone. 
\end{proposition}

\subsection{Monotone submodular functions}  
\label{sec:ms}

\xhdr{Overview} Our model for monotone submodular functions consists of a neural network which  cascades  the underlying  functions in a  recursive  manner for $N$ steps.  Specifically,
to compute the submodular function $F^{(n)}(\cdot)$ at step $n$,
it first linearly combines the submodular function $F^{(n-1)}(\cdot)$ computed in the previous step and a \emph{trainable} positive modular function  $m^{(n)}(\cdot)$  and then, applies a monotone concave activation function $\phi$ on it. 

\xhdr{Recursive model}  We model the submodular function $F_{\theta}(\cdot)$ as follows:
\begin{align}
% %\vspace{-2cm}
F^{(0)}(S) = m^{(0)} _{\theta}(S); \;
F^{(n)}(S) = \phi_{\theta}\big(\lambda  F^{(n-1)}(S)+  (1-\lambda)   m^{(n)} _{\theta}(S)\big); \;
F_{\theta}(S) = F^{(N)}(S);\label{eq:sub-model}
\end{align}
where the iterations are indexed by $1\le n\le N$, $\{m^{(n)}_{\theta}(\cdot)\}$ is   a sequence of positive modular functions, driven by a neural network with parameter $\theta$. $\lambda\in [0,1]$ is a tunable or trained parameter. We apply a linear layer with positive weights on the each feature vector $\feat_{\els}$ to compute the value of $m_{\theta} ^{(n)}(\cdot)$ and then compute $ m^{(n)} _{\theta}(S) = \sum_{\els \in S} m^{(n)} _{\theta}(\els)$.
% The underlying positive feature weights ensure that $m_{\theta} ^{(n)}(S) \ge 0$.  
Moreover, $\phi_{\theta}$ is an increasing concave function  which, as we shall see later, is modeled using neural networks. Under these conditions, one can  use Proposition~\ref{prop-basic}(i) to easily show that $F_{\theta}(S)$ is a monotone submodular function (Appendix~\ref{app:sec:alpha}).
% contains more discussions.

\subsection{Monotone-$\alpha$-submodular functions}
\label{sec:alpha}

Our characterization for submodular functions in Eq.~\eqref{eq:sub-model} are based on  Proposition~\ref{prop-basic}(i), which implies that a concave composed submodular function is submodular.
 However, to the best of our knowledge, a similar characterization of $\alpha$-submodular functions
is lacking in the literature. 
To address this gap, we first introduce a novel characterization of monotone $\alpha$-submodular functions and then use it to design a recursive model for such functions. 
 
\xhdr{Novel characterization of $\alpha$-submodular function} 
In the following, we show how we can characterize an \asb\ function
using a differential inequality (proven in Appendix~\ref{app:sec:alpha}). 
\begin{theorem}\label{thm:diffc}
Given the function $\diffc:\RR\to\RR^+$ and a modular function $m:V\to[0,1]$, the set function $F(S) = \diffc(\sum_{s\in S}m(s))$  is monotone $\alpha$-submodular for  $|S|\le k$, if $\diffc(x)$ is increasing in $x$ and
%\begin{align} \textstyle
$\frac{\d ^2 \diffc(x)}{\d x^2} \le \frac{1}{k}\log\left(\frac{1}{\alpha}\right)\frac{\d \diffc(x)}{\d x}$.
% \label{eq:diffc}
%\end{align}
\end{theorem}
The above theorem also implies that 
given $\alpha=1$, then $F(S) =  \diffc(\sum_{s\in S}m(s))$ is monotone submodular if $\diffc(x)$ is concave in $x$, which reduces to a particular case of Proposition~\ref{prop-basic}(i). 
Once we design $F(S)$ by applying $\diffc$ on a modular function, our next goal is to design more expressive and flexible modeling in a recursive manner similar to Eq.~\eqref{eq:sub-model}. To this end, we extend Proposition~\ref{prop-basic} to the case for \asb\ functions.

% formally state the our first characterization of monotone $\alpha$-submodular function (proven in Appendix~\ref{app:sec:alpha}).
\begin{proposition}
\label{prop:alpha1}
Given a monotone $\alpha$-submodular function $F(\cdot)$, $\phi(F(S))$ is monotone $\alpha$-submodular, if $\phi(\cdot)$  is an increasing concave function.  Here, $\alpha$ remains the same for $F$ and $\phi(F(\cdot))$.
\end{proposition}
Note that, when $F(S)$ is submodular, \ie, $\alpha=1$, the above result reduces to Proposition~\ref{prop-basic} in the context of concave composed modular functions.

\xhdr{Recursive model} Similar to Eq.~\eqref{eq:sub-model} for submodular functions, our model for \asb\ functions is also driven by an recursive model which maintains an \asb\ function and updates its value recursively for $N$ iterations.
However, in contrast to \our, where the underlying submodular function is initialized with a positive modular function,
we initialize the corresponding \asb\ function $F^{(0)}(S)$ with $\diffc_{\theta}(m_{\theta} ^{(0)}(S))$, where $\diffc_{\theta}$ is a trainable function satisfying
the conditions of Theorem~\ref{thm:diffc}. Then we recursively apply \todo{reviewers protested this is repetitive} a trainable monotone concave function on $F^{(n-1)}(S)$ for $n\in[N]$ to build a flexible model for \asb\ function. Formally, we have: 
% Specifically, given the ground set $V$ and a sequence of trainable positive modular functions $m_{\theta} ^{(n)} (\cdot) $,
% we model the $\alpha$-submodular function as follows:
\begin{align}
F^{(0)}(S) =    \diffc_{\theta}(m^{(0)} _{\theta}(S)); \;
F^{(n)}(S) = \phi_{\theta} \hspace{-0.3mm}\left(\hspace{-0.3mm} \lambda  F^{(n-1)}(S)+  (1-\lambda)    m^{(n)} _{\theta}(S)\right); \;
F_{\theta}(S) =  F^{(N)}(S);
\label{eq:alpha-sub-model}
%  \hspace{-1mm}& F^{(0)}(S) =    \diffc_{\theta}(m^{(0)} _{\theta}(S)),\label{eq:init-F-alpha}   \\
%   \hspace{-1mm}& F^{(n)}(S) = \phi_{\theta} \hspace{-0.3mm}\left(\hspace{-0.3mm} \lambda  F^{(n-1)}(S)+  (1-\lambda)    m^{(n)} _{\theta}(S)\right)\label{eq:F-alpha-rec} \\[-0.5ex]
%  \hspace{-1mm}& F_{\theta}(S) =  F^{(N)}(S) \label{eq:alpha-sub-model},
\end{align}
with $\lambda\in[0,1]$.  Here, $m_{\theta}$, $\diffc_{\theta}$ and $\phi_{\theta}$ are realized using neural networks parameterized by $\theta$.
% 
% Then, using Proposition~\ref{prop:alpha1} and Theorem~\ref{thm:diffc}, we can show that $F_{\theta}(S)$ is $\alpha$-submodular in $S$, as stated in the following (Proven in Appendix~\ref{app:sec:alpha}).
Then, using Proposition~\ref{prop:alpha1} and Theorem~\ref{thm:diffc}, we can show that $F_{\theta}(S)$ is $\alpha$-submodular in $S$ (proven in Proposition~\ref{prop:neural-alpha} in Appendix~\ref{app:sec:alpha}).

\subsection{Non-monotone submodular functions} 
\label{sec:nms}

In contrast to monotone set functions, non monotone submodular functions
can be built by applying concave function on top of \emph{only one modular function} rather than submodular function (Proposition~\ref{prop-basic} (i) vs.\ (ii)).
%
% The model given by \eqref{eq:init-F}--\eqref{eq:sub-model} above is very specific to monotone submodular functions. In general, such a model cannot be extended to a general form of submodular functions which also include   non-monotone functions, due to the presence of  the strong conditions in Proposition~\ref{prop-basic}(i). 
% However, as suggested in Proposition~\ref{prop-basic}(ii), one can provide a simple neural model for  general submodular functions, by applying any  non-monotone concave function on a modular function. 
To this end, we model a non-monotone submodular function $F_{\theta}(\cdot)$ as follows:
\begin{align}
\F(S)= \psi_{\theta}(m_{\theta}(S)) \label{eq:non-mon}
\end{align}
where $m_{\theta}(\cdot)$ is positive modular function but $\psi_{\theta}(\cdot)$ can be a \emph{non-monotone} concave function. Both $m_{\theta}$ and $\psi_{\theta}$ are trainable functions realized using neural networks parameterized by $\theta$. One can use Proposition~\ref{prop-basic}(ii) to show that $\F(S)$ is a non-monotone submodular function.
% Moreover, if $\psi_{\theta}(\cdot)$ is non-monotone (monotone) then $\F (S)$ is also non-monotone (monotone).    
 
\subsection{Neural parameterization of $m_\theta, \phi_\theta, \psi_\theta, \varphi_\theta$} 

We complete the neural parameterization introduced in Sections~\ref{sec:ms}--\ref{sec:nms}.  Each model has two types of component functions: (i)~the modular set function $m_{\theta}$; and, (ii)~the concave functions $\phi_{\theta}$ (Eq.~\eqref{eq:sub-model} and~\eqref{eq:alpha-sub-model}) and $\psi_{\theta}$ (Eq.~\eqref{eq:non-mon})  and the non-concave function $\diffc_{\theta}$ (Eq.~\eqref{eq:alpha-sub-model}).  While modeling $m_{\theta}$ is simple and straightforward, designing neural models for the other components, \ie, $\phi_{\theta}$, $\psi_{\theta}$ and $\diffc_{\theta}$  is non-trivial.
As mentioned before, because we cannot rely on the structural complexity of our `network' (which is a simple linear recursion) or a judiciously pre-selected library of concave functions \citep{bilmes2017deep}, we need to invest more capacity in the concave function, effectively making it universal.

\xhdr{Monotone submodular function}
Our model for monotone submodular function described in Eq.~\eqref{eq:sub-model} consists of two neural components: the sequence of modular functions $\setx{m^{(n)} _{\theta}}$ and $\phi_{\theta}$.

\ehdr{--- Parameterization of $m_{\theta}$} 
We model the modular function $m_{\theta}^{(n)}: 2^V \to \RR^+$ in Eq.~\eqref{eq:sub-model} as $m_\theta^{(n)}(S)=\sum_{s\in S} \theta \cdot \feat_\els$, where  $\feat_\els$ is the feature vector  for each element $\els\in S$ and both $\theta, \feat_\els$ are non-negative. \todo{@AD $\theta$ shared across all $n$?}

\ehdr{--- Parameterization of $\phi_{\theta}$} Recall that $\phi_{\theta} $ is an increasing concave function. We model it using the fact that a differentiable  function is concave if its second derivative is negative.  We focus on capturing the second derivative of the underlying function using a complex neural network of arbitrary capacity, providing a negative output. Hence, we use a positive neural network $\gt$ to model  the second derivative
\begin{align}
    \frac{\d^2\phi_{\theta}(x)}{\d x^2} = -\gt(x)\le 0.
\end{align}
Now, since $\phi_{\theta}$ is increasing, we have:
\begin{align}
    \frac{\d \phi_{\theta}(x)}{\d x} &=
    \int_{b=x}^ {b=\infty} \gt(b) \, \d b    \ge 0 
    % \begin{align}
\implies \phi_{\theta}(x) = \int_{a=0} ^{a=x} \int_{b=a} ^{b=\infty} \gt(b) \, \d b \, \d a.   \label{eq:phim}
\end{align}
% which finally leads to the following model for $ \phi_{\theta}(x) $
% \begin{align}
%  \phi_{\theta}(x) = \int_{a=0} ^{a=x} \int_{b=a} ^{b=\infty} \gt(b) \, \d b \, \d a.   \label{eq:phim}
% \end{align}
Here, $\phi_{\theta}(\cdot)$ is normalized, \ie, $\phi_{\theta}(0)=0$, which ensures that $\{F_{\theta}^{(n)}\}$ in Eq.~\eqref{eq:sub-model} are also  normalized. An offset in~Eq.~\eqref{eq:phim}  allows a nonzero initial value of $\phi_{\theta}(\cdot)$, if  required.
Note that  monotonicity and concavity of $\phi_{\theta}$ can be achieved by the restricting positivity of $\gt$.
Such a constraint  can be ensured by setting $\gt(\cdot)=\relu(\Lambda^{(h)} _{\theta}(\cdot))$, where $\Lambda^{(h)} _{\theta}(\cdot)$ is any complex neural network.  Hence, $\phi_{\theta}(\cdot)$ represents a class of universal approximators of normalized increasing concave functions, if $\Lambda^{(h)}  _{\theta}(\cdot)$ is an universal approximator of continuous functions~\cite{hornik1989multilayer}. For a monotone submodular function of the form of concave composed modular function, we have the following result (Proven in Appendix~\ref{app:sec:alpha}).
\begin{proposition}
\label{prop:ua}
Given an universal set $V$, a constant $\epsilon>0$ and a submodular function $F(S) = \phi(\sum_{s\in S} m(\zb_s))$  where $\zb_s \in \RR^d$, $S\subset V$, $0 \le m(\zb) < \infty$ for all $\zb \in \RR^d$. Then there exists two fully connected neural networks   $m_{\theta_1}$  and $h_{\theta_2} $   of width $d+4$ and $5$ respectively, each with ReLU activation function, such that the following conditions hold:
\begin{align}
   \left\|F(S) -   \int_{a=0} ^{a=\sum_{s\in S}m_{\theta_1}(\zb_s)} \int_{b=a} ^{b=\infty} h_{\theta_2}(b) \, \d b \, \d a. \right\|    \le \epsilon \quad \forall \ S\subset V\label{eq:ua}
\end{align}
\end{proposition}

\xhdr{Monotone \asb\ model} An \asb\ model described in Eq.~\eqref{eq:alpha-sub-model} has three trainable components: (i)~the sequence of modular functions $\{m_{\theta} ^{(n)}(\cdot)\}$, (ii)~the concave function $\phi_{\theta}(\cdot)$ and (iii)~$\diffc_{\theta}(\cdot)$. For the first two components, we reuse the parameterizations used for monotone submodular functions.
In the following, we describe our proposed neural parameterization of $\diffc_{\theta}$.
% For the modular function $m_{\theta}$ and the concave function $\phi_{\theta}$, we use the same parameterization as in 

\ehdr{--- Parameterization of $\diffc_{\theta}(\cdot)$}
From Theorem~\ref{thm:diffc}, we note that $\ttc(\cdot)$ is increasing and satisfies 
% \begin{align}
  $\frac{\d ^2 \ttc(x)}{\d x^2} \le \ga\frac{\d  \ttc(x)}{\d x}$
%   \label{eq:diffcx}  
% \end{align}
where, $\ga = \frac{1}{k}\log\left( {1}/{\alpha}\right)$. 
% Eq.~\eqref{eq:diffcx}
It implies that
\begin{align}
&  e^{-x \ga} \frac{\d ^2 \ttc(x)}{\d x^2}-  \ga e^{-x \ga}\frac{\d  \ttc(x)}{\d x} \le 0  \implies \frac{\d}{\d x}\left(e^{-x \ga} \frac{  \d  \ttc(x) }  {\d x} \right) \le 0
\end{align}
Driven by the last inequality, we have
% \begin{align}
% \frac{ e^{-x \ga} \d  \ttc(x) } {\d x} = \int_{x} ^{\infty} \galphat(b)\, \d b  \label{eq:diffc_first_derivative}
% \end{align}
\begin{align}
e^{-x \ga}\frac{  \d  \ttc(x) } {\d x} = \int_{x} ^{\infty} \galphat(b)\, \d b 
\implies \ttc(x) =   \int_{a=0} ^{a=x}   e^{a\ga  } \int_{b= a} ^{b= \infty} \galphat(b)\, \d b\,  \d a \label{eq:phiam}
% \label{eq:diffc_first_derivative}
\end{align}
% where $\galphat(\cdot) \ge 0$. Note that, we could have modeled the above quantity using a negative integrand in RHS, \eg,
% $\int_0 ^x - \galphat(b) \, \d b $. However, recall from Theorem~\ref{thm:diffc} that $\d \ttc(x)/\d x \ge 0$. Therefore,
% we resort to the modeling choice in Eq.~\eqref{eq:diffc_first_derivative}, which ensures that $\ttc(\cdot)$ is increasing. Finally, 
% %
% Eq.~\eqref{eq:diffc_first_derivative} gives us the following model for $\ttc (x)$.
% \begin{align}
% \label{eq:phiam}
%   \ttc(x) =   \int_{a=0} ^{a=x}   e^{a\ga  } \int_{b= a} ^{b= \infty} \galphat(b)\, \d b\,  \d a
% \end{align}

\xhdr{Parameterizing non-monotone submodular model} As suggested by Eq.~\eqref{eq:non-mon}, 
our model for non-monotone submodular function contains a non-monotone concave function $\psi_{\theta}$ and a modular function $m_{\theta}$. We model $m_{\theta}$ using the same parameterization used for the monotone set functions. We parameterize the $\psi_{\theta}$ as follows.

\ehdr{--- Parameterization of $\psi_{\theta}$} Modeling a  generic (possibly non-monotone) submodular function requires a general form of concave function $\psi_{\theta}$ which is not necessarily increasing.  The trick is to design $\psi_{\theta}(\cdot)$ in such a way that  its second derivative is negative everywhere, whereas its first derivative can have any sign. 
For $x \in [0,x_{\max}]$, we have:
\begin{align}
 \psi_{\theta}(x)&  = \int_{a=0} ^{a=x} \int_{b=a} ^{b=\infty} \gt(b) \, \d b \, \d a    - \int_{a=x_{\max}-x} ^{a=x_{\max}} \int_{b=a} ^{b=\infty}  \ggt(b) \, \d b \, \d a, \label{eq:psim}
\end{align}
where $\gt, \ggt \ge 0$. Moreover, we assume that $\int_0^\infty \gt (b)\, \d b$ and $\int_0^\infty \ggt  (b)\, \d b$    are convergent.
We use $x_{\max}$ in the upper limit of the second integral to ensure
that $\psi_{\theta}$ is normalized, \ie, $\psi_{\theta}(0)=0$.
Next, we compute the first derivative of $\psi_{\theta}(x)$ as:
\begin{align}
    \frac{\d \psi_{\theta}(x)}{\d x} &=
    \int_{b=x}^ {b=\infty} \gt(b) \, \d b   - \int_{b=x_{\max }-x} ^{b=\infty} \ggt(b) \, \d b,
\end{align}
which can have any sign, since both integrals are positive. Here, the second derivative of $\psi_{\theta}$ becomes
\begin{align}
    \frac{\d^2\psi_{\theta}(x)}{\d x^2} = -  \gt(x) -  \ggt(x_{\max}-x)  \le 0 \label{eq:psimx}
\end{align}
which implies that $\psi_{\theta}$ is concave. 
Similar to $\gt(\cdot)$, we can model  $\ggt(\cdot) = \relu(\Lambda^{g} _{\theta} (\cdot))$. Such a representation makes $\psi_{\theta}(\cdot)$   a universal approximator of normalized concave functions. As suggested by Eq.~\eqref{eq:non-mon}, $\psi_{\theta}$ takes $m_{\theta}(S)$ as input. Therefore, in practice, we set
% \begin{align}
$x_{\max} = \max_S m_{\theta}(S)$.
% \end{align}
 
\subsection{Parameter estimation from (\tset, \tvalue) pairs}
\label{sec:sec:train}
In this section, our goal is to learn
$\theta$ from a set of pairs $\set{(S_i,y_i)\given i\in [I]}$,  such that $y_i\approx F_{\theta}(S_i)$. Hence, our task is to solve the following optimization problem:
\begin{align}
\label{eq:loss}
\textstyle \min_{\theta} \sum_{i\in[I]}(y_i-F_{\theta}(S_i))^2
\end{align}
In the following, we discuss methods to solve this problem.

\xhdr{Backpropagation through double integral}
Each of our proposed set function models consists of one or more double integrals. 
Thus computation of the gradients of the loss function in Eq.~\eqref{eq:loss} requires gradients of these integrals.  To compute them,  we leverage the methods proposed by \citet{umnn}, which specifically provide forward and backward procedures for neural networks involving integration. Specifically, we use the Clenshaw-Curtis (CC) quadrature or Trapezoidal Rule for numerical computation of $\phi_{\theta}(\cdot)$ and $\psi_{\theta}(\cdot)$. On the other hand, we compute the gradients $\nabla_{\theta}\, \phi_{\theta}(\cdot)$ and $\nabla_{\theta}\, \psi_{\theta}(\cdot)$ by leveraging Leibniz integral rule~\cite{flanders1973differentiation}. In practice, we replace the upper limit ($b=\infty$) of the inner integral in Eqs.~\eqref{eq:phim},~\eqref{eq:phiam} and~\eqref{eq:psim} with  $b_{\max} =x_{\max}$ during training.

\xhdr{Decoupling into independent integrals}
The aforementioned end-to-end training can be challenging in practice, since the gradients are also integrals which can make loss convergence elusive. Moreover, it is also inefficient, since for each step of CC quadrature of the outer integral, we need to perform numerical integration of the entire inner integral. To tackle this limitation, we decouple the underlying double integral into two single integrals, parameterized using two neural networks.

\ehdr{--- Monotone submodular functions} During training our monotone submodular function model in Eq.~\eqref{eq:sub-model}, we model $\phi_{\theta}$ in Eq.~\eqref{eq:phim} as:
\begin{align}
    \phi_{\theta}(x) & = \int_{0} ^{x} \phi' _{\theta}(a) \, \d a,  \quad
     \phi' _{\theta}(x) = \int_{x} ^{\infty} h _{\beta}(a) \, \d a \label{eq:decouple-mon-sub}
\end{align}
In contrast to end-to-end training of $\theta$ where $\phi_{\theta}$ was realized only via neural network of $h_{\theta}$,
here we use two decoupled networks, \ie, $\phi' _{\theta}$ and $h_{\beta}$. Then, we learn 
$\theta$ and $\beta$ by minimizing the regularized sum of squared error, \ie,
\begin{align}
\min_{\theta,\beta} \sum_{i\in[I]} \sum_{n\in[N]}
\Big[ \reg \Big(\phi' _{\theta}(G^{(n)} (S_i))-\int_{G^{(n)}(S_i)} ^{\infty}
  h_{\beta}(a) \, \d a\Big)^2   
 + \big(y_i-F_{\theta}(S_i)\big)^2 \Big].  \label{eq:loss-mon-decoup}
\end{align}
Here, $G^{(n)}(S) = \lambda F^{(n-1)}(S) + (1-\lambda) m^{(n)} _{\theta}(S)$ is the input to $\phi_{\theta}$ in the recursion~\eqref{eq:sub-model}.
Recall that $N$ is the number of steps in the recursion.  Since $\phi_{\theta}$ is monotone, we need $\phi'_{\theta} \ge 0$ which is ensured by 
$\phi' _{\theta}(x) = \relu(\Lambda^{\phi'} _{\theta} (x))$.
In principle, we would like to have $ \phi' _{\theta}(x) = \int_{x} ^{\infty} h _{\beta}(a) \, \d a$ for all $x\in \RR$. In practice, we approximate this by penalizing the regularizer values in the domain of interest---the values which are fed as input to  $\phi_{\theta}$. 

While there are potential hazards to such an approximation, in our experiments, the benefits outweighed the risks.
The above parameterization involves only single integrals. The use of an auxiliary network $h_{\beta}$ allows more flexibility, leading to improved robustness during training optimization.
The above minimization task achieves approximate concavity of $\phi$ via training, whereas backpropagating through double integrals enforces concavity by design.
Appendix~\ref{app:decoupling} extends this method for \asb\ and non-monotone submodular functions.

% \input{050summarization}

% In many applications, the goal is to select an optimal subset from a universal set. For example, in product recommendation, given all the items, one may like to learn which subset of recommended items maximizes the utility from the user's perspective; in data summarization, given a collection of photographs, a user may choose a few representative samples to disclose   to friends. 
% then we present the probabilistic greedy model. Finally, we  
%  present  our proposed method which trains this greedy model so that the trained model becomes invariant to the order of input elements.

\section{{Differentiable subset selection}}
\label{sec:summarization}

% In Section~\ref{sec:sec:train}, we tackled the problem of learning $F_{\theta}$ from (\tset, \tvalue) pairs.  In this section, we consider learning $F_{\theta}$ from a given set of (\tuniv, \thigh) pair instances.
% Formally, the training set $\Ucal$ consists of $\setx{(V,\St)}$ pairs, where $V$ is some arbitrary subset of the universe of element, and $\St \subseteq V$ is a high-value subset of~$V$.
% Through this distant (compared to direct values) supervision, the system learns an underlying 
% function $F_{\theta}$.  This set function then drives the sampling of high-value sets from test instances of \tuniv{}s. 
In Section~\ref{sec:sec:train}, we tackled the problem of learning $F_{\theta}$ from (\tset, \tvalue) pairs.  In this section, we consider learning $F_{\theta}$ from a given set of (\tuniv, \thigh) pair instances.
% Formally, the training set $\Ucal$ consists of $\setx{(V,\St)}$ pairs, where $V$ is some arbitrary subset of the universe of element, and $\St \subseteq V$ is a high-value subset of~$V$.
% Through this distant (compared to direct values) supervision, the system learns an underlying 
% function $F_{\theta}$.  This set function then drives the sampling of high-value sets from test instances of \tuniv{}s. 

\subsection{Learning $F_{\theta}$ from (\tuniv, \thigh)}

Let us assume that the training set $\Ucal$ consists of $\setx{(V,\St)}$ pairs, where $V$ is some arbitrary subset of the universe of element, and $\St \subseteq V$ is a high-value subset of~$V$.
Given a set function model $F_{\theta}$, our goal is to estimate $\theta$ so that,
across all possible cardinality constrained subsets in $S'\subseteq V$
\todo{may result in protest: who gives this at test time and how?} with $|S'|=|S|$,
$F_{\theta}(\cdot)$ attains its maximum value at~$\St$.
Formally, we wish to estimate $\theta$ so that,
for all possible $(V,\St)\in \Ucal$, we have:
\begin{align}
\St &= \textstyle \argmax_{S' \subset V}\F(S'), \ \text{subject to} \ |S'| = |\St|. 
\label{eq:pr-st}
\end{align}

\subsection{Probabilistic greedy model}
\label{sec:sec:prob-greedy}

The problems of maximizing  both monotone submodular and monotone \asb\ functions rely on a greedy heuristic~\cite{nemhauser1978analysis}.
It sequentially chooses elements maximizing the marginal gain and hence, cannot directly support backpropagation.  \citet{tschiatschek2018differentiable} tackle this challenge with a probabilistic model which greedily samples elements
from a softmax distribution with the marginal gains as input. Having chosen the first $j$ elements of \thigh{} from \tuniv{} $V$, it draws the $(j+1)^{\text{th}}$ element from the remaining candidates in $V$ with probability proportional to the marginal utility of the candidate. If the elements of $S$ under permutation $\pi$ are written as $\pi(S)_j$  for $j\in[|S|]$, then the probability of selecting the elements of the set $S$ in the sequence $\pi(S)$ is given by
\begin{align}
 &\Pr{}_\theta (\pi(S)\given V) 
=\prod_{j=0} ^{|S|-1} \frac{\exp(\tau \F (\pi(S)_{j+1} \given \pi(S)_{\le j}))}{\sum_{s\in V\cp \pi(S)_{\le j}}
 \exp(\tau\F (s \given \pi(S)_{\le j}))} \label{eq:greedy}
\end{align}
Here, $F_{\theta}$ is the underlying submodular  function, $\tau$ is a temperature parameter and $S_{0} = \emptyset$.

\xhdr{Bottleneck in estimating $\theta$}  Note that, the above model~\eqref{eq:greedy} generates the elements in a sequential manner and is sensitive to~$\pi$.
\citet{tschiatschek2018differentiable} attempt to remove this sensitivity by maximizing the sum of probability \eqref{eq:greedy} over all possible~$\pi$:
\begin{align}
\theta^* &=
\argmax_{\theta}\sum_{(V,\St)\in\Ucal} 
\log \sum_{\pi\in \Pi_{|\St|}} \Pr{}_\theta (\pi(S) \given V).  \label{eq:sumperm}
\end{align}
However, enumerating all such permutations for even a medium-sized subset is expensive both in terms of time and space. They attempt to tackle this problem by approximating the mode and the mean of $\Pr (\cdot)$ over the permutation space $\Pi_{|\St|}$. But, it still require searching over the entire permutation space, which is extremely time consuming.

\subsection{Proposed approach}
Here, we describe our proposed method of permutation adversarial parameter estimation that avoids enumerating all possible permutations of the subset elements, 
while ensuring that the learned parameter $\theta^*$ remains permutation invariant. 

\xhdr{Max-min optimization problem} We first set up a max-min game between a permutation generator and
the maximum likelihood estimator (MLE) of $\theta$, similar to other applications~\cite{permgnn,prateek}.  Here, for each subset $S$, the permutation generator produces an adversarial permutation $\pi \in \Pi_{|S|}$
which induces a low value of the underlying likelihood. On the other hand,  MLE learns $\theta$ in the face of these adversarial permutations. Hence, we have the following optimization problem:
\begin{align}
 \max_{\theta} \min_{\pi \in \Pi_{|\St|}} \textstyle \sum_{(V,\St)\in \Ucal}
 \log \Pr{}_\theta (\pi(S) \given V) \label{eq:hard-maxmin}  
\end{align}

\xhdr{Differentiable surrogate for permutations}
Solving the inner minimization in \eqref{eq:hard-maxmin} requires searching for $\pi$ over $\Pi_{|S|}$, which appears to confer no benefit beyond~\eqref{eq:sumperm}. To sidestep this limitation, we relax the inner optimization problem by using a doubly stochastic matrix $\Pb \in \mathcal{P}_{|S|}$ as an approximation for the corresponding permutation~$\pi$.
Suppose $\Feat_S = [\feat_s]_{s\in \St}$ is the feature matrix where each row corresponds to an element of~$S$.  Then $\Feat_{\pi(S)} \approx \Pb \Feat_S$. 
Thus, $\Pr_\theta(\pi(S)|V)$ can be evaluated by iterating down the rows of $\Pb\Feat_S$ and, therefore, written in the form $\Pr_\theta(\Pb, S)$.    Thus, we get a continuous approximation to \eqref{eq:hard-maxmin}:
\begin{align}
\max_{\theta} \min_{\Pb \in \mathcal{P}_{|S|} }
\textstyle\sum_{(V,S)\in \Ucal} \log \Pr{}_\theta (\Pb, S \given V) \label{eq:soft-maxmin}  
\end{align}
so that we can learn $\theta$ by continuous optimization.
We generate the soft permutation matrices $\Pb$ using a Gumbel-Sinkhorn network~\cite{mena2018learning}. Given a subset $S$, it takes a seed matrix $\Bb_S$ as input and generates $\Pb^S$ in a recursive manner:
\begin{align}
 \Pb^{0} = \exp(\Bb_S/t); \quad
 \Pb^{k} = \DD_c \left(\DD_r\left(\Pb^{(k-1)} \right)\right) \label{eq:sinkh}
\end{align}
Here, $t$ is a temperature parameter; and, $\DD_c$ and $\DD_r$ provide column-wise and row-wise normalization. Thanks to these two operations, $\Pb^k$ tends to a doubly stochastic matrix for sufficiently large $k$.  Denoting $\Pb^\infty = \lim_{k\to\infty} \Pb^k$, one can show~\cite{mena2018learning} that
\begin{align*}
 \Pb^\infty = \argmax_{\Pb \in \mathcal{P}_{|S|}}
 \text{Tr} (\Pb^\top \Bb_S) - t \sum_{i,j} \Pb(i,j) \log \Pb(i,j)  
\end{align*}
where the temperature $t$ controls the `hardness' of the ``soft permutation'' $\Pb$.
As $t\to 0$, $\Pb$ tends to be a hard permutation matrix~\cite{cuturi2013sinkhorn}. 
The above procedure demands different seeds $\Bb_{\St}$ across different pairs of $(V, \St)\in \Ucal$, which makes the training  computationally expensive in terms of both time and space. Therefore, we model $\Bb_{\St}$ by feeding the feature matrix $\Feat_S$ into an additional neural network $G_{\omega}$ which is shared across different $(V,S)$ pairs. Then the optimization~\eqref{eq:soft-maxmin} reduces to $\max_{\theta} \min_{\omega}\sum_{(V,S)\in \Ucal}\log \Pr_\theta (\omega, S)$, with $\omega$ taking the place of~$\Pb$.

% \input{060expSyn}

% %%\vspace{-1mm}
\section{Experiments}
%%\vspace{-2mm}
\label{sec:expt}
In this section, we first evaluate \our\  using a variety of synthesized set functions and show that it is able to fit them under the supervision of $(\tset,\tvalue)$ pairs more accurately than the state-of-the-art methods.
% These experiments are consistent with recent work on learning submodular functions, except that our planted functions are more challenging.
% Appendix~\ref{app:addl}
% contains additional experiments.
Next, we use datasets gathered from Amazon baby registry ~\cite{gillenwater2014expectation} to show that \our\ can
learn to select data from  (\tuniv, \thigh) pairs more effectively than several baselines. Our code is in \href{https://tinyurl.com/flexsubnet}{https://tinyurl.com/flexsubnet}.

% %\vspace{-2mm}
\subsection{Training by (\tset, \tvalue) pairs}
% %\vspace{-1mm}
\label{sec:syn-setup}
\xhdr{Dataset generation} We generate $|V|{=}10^4$ samples, where we draw the feature vector $\feat_\els$ for each sample $\els{\in}V$ uniformly \todo{rev may claim these are basically random IDs therefore uninteresting} at random, \ie,
$\feat_\els{\in}\text{Unif}[0,1]^{d}$ with $d{=}10$. Then, we generate subsets $S$ of different sizes by randomly gathering elements from the set~$V$. 
Finally, we compute the values of $F(S)$ for different submodular functions $F$.
\todo{@AD: dangers of log(negative) how handled} 
Specifically, we consider seven planted set functions:
\begin{enumerate*}[(i)]
\item \textbf{Log}: $F(S) {=} \log (\sum_{\els\in S} \bm{1}^\top \feat_\els)$,

\item  \textbf{LogDet}: $F(S) = \log \text{det}(\II+\sum_{\els\in S}\feat_\els \feat_\els ^\top)$~\cite{shamaiah2010greedy,borodin2009determinantal},

\item \textbf{Facility location (FL)}: $F(S)= \sum_{s'\in V} \max_{\els\in S}  
\feat_\els^\top \feat_{s'}/ (||\feat_\els|| ||\feat_{s'}||)$~\cite{mirchandani1990discrete,frieze1974cost,du2012primal,ortiz2015multi},

\item \textbf{Monotone graph cut (\gcutm)}: 
In general, \gcut\ is computed using $F(S):= \sum_{u \in V, v \in S} \feat_u ^\top \feat_v $ $  - \tradeoff \sum_{u, v \in S} \feat_u ^\top \feat_v$. It measures the weighted cut across $(S, V\cp S)$ when the weight of the edge $(u,v)$ is computed as $\feat_u ^\top \feat_v$~\cite{iyer2021submodular,schrijver2003combinatorial,jegelka2009notes,jegelka2010cooperative,jegelka2011submodularity}. Here, $\tradeoff$ trades off between diversity and representation. We set $\tradeoff=0.1$. Note that \gcut\ is monotone (non-monotone) submodular when $\tradeoff{<}0.5$ ($\tradeoff{>}0.5)$. 

\item \textbf{\logxsqrt}: $F(S)$ ${=}$ $[\log (\sum_{\els\in S} \bm{1}^\top \feat_\els)]{\cdot}[\sum_{\els \in S} \bm{1}^\top\feat_\els]^{1/2}$,

\item \textbf{\logxdet}: $F(S) = [\log (\sum_{\els\in S} \bm{1}^\top \feat_\els)]\cdot[\log \text{det}(\II+\sum_{\els\in S}\feat_\els \feat_\els ^\top)]$ and

\item \textbf{Non monotone graph cut (\gcutn)}: It is the graph cut function in (iv) with $\tradeoff=0.8$. 
\end{enumerate*}
Among above set functions, (i)--(iv) are monotone submodular functions, (v)--(vi) are monotone \asb\ functions and (vii) is a non-monotone  submodular function. We set the number of steps in the recursions~\eqref{eq:sub-model},~\eqref{eq:alpha-sub-model} as $N=2$.

% This constitutes a diverse set of latent submodular functions.
% Among them, \textbf{Log} and \textbf{sqrt} are obtained by applying simple concave functions on a modular function, \textbf{Facility location} models  coverage, \textbf{LogDet} models diversity and \textbf{Graph cut} models network flow. 

\xhdr{Evaluation protocol}
% To evaluate the generalizability  of \our, we
% present a collection of synthetically generated pairs $\set{(S, F(S))}$ as input to \our\ which
% aims to approximate $F(S)$ using the underlying neural network $F_{\theta}(S)$. 
% %%
% Specifically, we  train $\theta$ by minimizing the squared error $\sum_{S\in \Scal_{\text{train}}}(F(S)-F_{\theta}(S))^2$ on the training set $\Scal_{\text{train}}$ and then
% predict the value of $F(T)$ on the test sets $T\in\Scal_{\text{test}}$. 
We sample $|V|{=}10000$ (\tset,\tvalue) instances as described above 
and split them into train, dev and test folds of equal size. We present the train and dev folds to the set function model and measure the performance in terms of RMSE on test fold instances.
In the first four datasets, we used our monotone submodular model described in Eq.~\eqref{eq:sub-model}; in case of \logxsqrt\ and \logxdet, we used our \asb\ model described in Eq.~\eqref{eq:alpha-sub-model}; and, for \gcutn, we used our non-monotone submodular model in Eq.~\eqref{eq:non-mon}.
For \asb\ model, we tuned $\alpha$ using cross validation.
Appendix~\ref{app:syn} provides hyperparameter tuning details for all methods.

\xhdr{Baselines} We compare \our\ against four state-of-the-art models, \viz,
\settx~\cite{lee2019set}, \deepset~\cite{zaheer2017deep}, deep submodular function (DSF)~\cite{bilmes2017deep} and mixture submodular function (SubMix)~\cite{tschiatschek2014learning}. In principle, set transformer shows high expressivity due to its ability to effectively incorporate the interaction between elements.
No other method including \our{} is able to incorporate such interaction. Thus, given sufficient network depth and width,
\settx\ should show higher accuracy than any other method. However, \settx\ consumes significantly higher GPU memory even with a small number of parameters. Therefore, to have a fair comparison with the rest non-interaction methods, we kept the parameters of \settx\ to be low enough so that it consumes same amount of GPU memory, as with other non-interaction based methods.

% We also compare \our with its variant where the number of  steps in RNN $N=1$. 
% Appendix~\ref{app:syn} contains more details about the baselines.

% \best{0.015 $\pm$ 0.000} & \best{0.013 $\pm$ 0.000} & \best{0.022 $\pm$ 0.000} \\
% {0.078 $\pm$ 0.001} & {0.073 $\pm$ 0.001} & {0.078 $\pm$ 0.001}
\begin{table*}[t]
\centering
\small
\adjustbox{max width=0.95\hsize}{ \tabcolsep 5pt
\begingroup \footnotesize
\begin{tabular}{l||c|c|c|c|c|c|c}
\hline 
& Log                      & LogDet            & FL           & \gcutm      & \logxsqrt & \logxdet & \gcutn  \\ \hline \hline 
% \our & \best{0.015 $\pm$ 0.013}  & \best{0.013 $\pm$ 0.007}  & \best{0.022 $\pm$ 0.021}  & \best{0.004 $\pm$  0.002}  & \best{0.032 $\pm$ 0.025}  & \best{0.025 $\pm$  0.019} & \best{0.068 $\pm$ 0.046}  \\ 
% %
% \settx & \secbest{0.060 $\pm$ 0.056}  & \secbest{0.029 $\pm$ 0.015}  & \secbest{0.063 $\pm$ 0.057 } & \secbest{0.014 $\pm$ 0.009 } & \secbest{0.037 $\pm$ 0.019 } & \secbest{0.051 $\pm$ 0.042 } &{0.171 $\pm$ 0.106}  \\
% %
% \deepset & 0.113 $\pm$ 0.100  & 0.044 $\pm$ 0.026  & 0.179 $\pm$ 0.152  & 0.058 $\pm$ 0.038  & 0.079 $\pm$ 0.050  & 0.070 $\pm$ 0.046   & \secbest{0.075 $\pm$ 0.046}  \\
% %
% DSF & 0.258 $\pm$ 0.179  & 0.684 $\pm$ 0.437  & 0.189 $\pm$ 0.146  & 0.778 $\pm$ 0.523  & 0.243 $\pm$ 0.153  & 0.274 $\pm$ 0.200  & 0.970 $\pm$ 0.588  \\
% %
% SubMix & 0.148 $\pm$ 0.117  & 0.063 $\pm$ 0.031  & 0.172 $\pm$ 0.125  & 0.018 $\pm$ 0.016 & 0.158 $\pm$ 0.103  & 0.162 $\pm$ 0.106 & 1.722 $\pm$ 0.846 \\

\our& \best{0.015 $\pm$ 0.000}  & \best{0.013 $\pm$ 0.000}  & \best{0.022 $\pm$ 0.000 } & \best{0.004 $\pm$ 0.000} & \best{0.032 $\pm$ 0.000 } & \best{0.025 $\pm$ 0.000} & \best{0.068 $\pm$ 0.001} \\ 
\settx & \secbest{0.060 $\pm$ 0.001}  & \secbest{0.029 $\pm$ 0.000 } & \secbest{0.063 $\pm$ 0.001  }& \secbest{0.014 $\pm$  0.000 }& \secbest{0.037 $\pm$ 0.000}  & \secbest{0.051 $\pm$ 0.001} & {0.171 $\pm$ 0.002} \\ 
\deepset & 0.113 $\pm$ 0.002  & 0.044 $\pm$ 0.000  & 0.179 $\pm$ 0.003  & 0.058 $\pm$ 0.001 & 0.079 $\pm$ 0.001  & 0.070 $\pm$ 0.001 & \secbest{0.075 $\pm$ 0.001} \\ 
DSF& 0.258 $\pm$ 0.003  & 0.684 $\pm$ 0.007  & 0.189 $\pm$ 0.003  & 0.778 $\pm$ 0.009 & 0.240 $\pm$ 0.003  & 0.274 $\pm$ 0.003 &0.970 $\pm$ 0.010 \\ 
SubMix &0.148 $\pm$ 0.002  & 0.063 $\pm$ 0.001  & 0.172 $\pm$ 0.002  & 0.018 $\pm$ 0.000 & 0.158 $\pm$ 0.002  & 0.154 $\pm$ 0.002 &1.722 $\pm$ 0.015 \\
\hline 
\end{tabular}
\endgroup}
\caption{Performance measured in terms of RMSE on synthetically generated examples using several set functions. %
% \viz, Log, LogDet, Facility Location (FL)
% and Graph cut (\gcutm), \logxsqrt, \logxdet\ and \gcutn\ for \our, \settx~\cite{lee2019set}, \deepset~\cite{zaheer2017deep}, deep submodular function (DSF)~\cite{bilmes2017deep} and mixture submodular function (SubMix)~\cite{tschiatschek2014learning}. Among them, the first four datasets are monotone submodular functions,
% \logxsqrt\ and \logxdet\  are \asb\ functions and 
% \gcutn\ is a non-monotone submodular function.
Number in \textbf{bold} font (\sec{underline}) indicate the best (second best) performers.}
\label{tab:syn-table}
% %\vspace{-1mm}
\end{table*}

% \subsection{Results}

\xhdr{Results}
We compare the performance of \our\ against the above four baselines
in terms of RMSE on the test fold. 
Table~\ref{tab:syn-table} summarizes the results.
We make the following observations. 
\begin{enumerate*}[(1)]
% [wide, labelwidth=*, labelindent=0pt, leftmargin=0pt]
\item \our\ outperforms the baselines by a substantial margin across all datasets.
\item Even with reduced number of parameters, \settx\ outperforms all the other baselines. 
% While both \deepset\ and \settx\ 
% ensure permutation invariant aggregation of the feature vectors $\set{\zb_s\given s\in S}$, 
\settx\ allows explicit interactions between set elements via all-to-all attention layers. As a result, it outperforms all other baselines even without explicitly using the knowledge of submodularity in its network architecture.
Note that, like \deepset, \our\ does not directly model any interaction between the elements.
However, it explicitly models  submodularity or $\alpha$-submodularity in the network architecture, which helps it outperform the baselines. We observe improvement of the performance of \settx, if we allow more number of parameters (and higher GPU memory). 
% Fixing the underlying concave function which may not match the true one severely affects training. 
% (ii) \deepset outperforms all other baselines. Since \deepset provides a universal set function approximator, it is able to learn the underlying function
% without explicitly using the knowledge of submodularity in the network architecture. 
%
%\item On an average, the performance of DSF is particularly poor.  Although in theory, DSF can capture a wide variety of submodular functions, fixing the concave function prevents it from realizing its full potential during data-driven learning.
%
\item While, in principle, DSF function family contains that of \our, standard multilayer instantations of DSF with fixed concave functions cannot learn well from (\tset,\tvalue) training.
\end{enumerate*}

 %\vspace{-2mm}
\subsection{Training by (\tuniv, \thigh)}
\label{sec:real-main}

\xhdr{Datasets} We use the Amazon baby registry dataset \cite{gillenwater2014expectation} which contains 17 product categories. Among them, we only consider those categories where $|V|>50$, where $V$ is the total number of items in the universal set.
These categories are:
\begin{enumerate*}[(i)]
\item Gear, \item Bath, \item Health, \item Diaper, \item Toys,
\item Bedding, \item Feeding, \item Apparel and \item Media.
\end{enumerate*}
They are also summarized in Appendix~\ref{app:real}. 

\xhdr{Evaluation protocol}
Each dataset contains a universal set $V$ and a set of subsets $\Scal=\set{S}$.
Each item $s$ is characterized by a short textual description such as ``bath: Skip Hop Moby Bathtub Elbow Rest, Blue : Bathtub Side Bumpers : Baby''.
From this text, we compute $\feat_s$ using BERT~\cite{devlin2018bert}.
We split $\Scal$ into equal-sized training ($\Scal_{\text{train}}$), dev ($\Scal_{\text{dev}}$) and test ($\Scal_{\text{test}}$) folds.
We disclose the training and dev sets to the submodular function models, which are trained using our proposed permutation adversarial approach (Section~\ref{sec:summarization}).  Then  the trained model is used to sample item sequence $S'$ with prefix $S'_{\le K}=\{s_1,\ldots,s_K\}$.  Here, $s_i$ is the item selected at the $i^{\text{th}}$ step of the greedy algorithm.
We assess the quality of the sampled sequence using two metrics.

\begin{table*}[t]
\centering
\adjustbox{max width=0.86\hsize}{ \tabcolsep 3pt
\begingroup \footnotesize
\begin{tabular}{p{1.1cm}||c|ccccc|c|ccccc}
\hline
 
& \multicolumn{6}{c|}{\textbf{Mean Jaccard Coefficient (MJC)}} & \multicolumn{6}{c}{\textbf{Mean NDCG@10}} \\
&\our  & DSF & SubMix & FL  & DPP & DisMin 
&\our  & DSF & SubMix & FL  & DPP & DisMin   \\ \hline\hline
Gear & \textbf{0.101} & \sec{0.099}  & 0.028 & 0.019 & 0.014 & 0.013  &  \textbf{ 0.539 } &  \sec{ 0.538 } &   0.449 &   0.433 &   0.425 &   0.426 \\ 
Bath & \textbf{0.091} & \sec{0.087}  & 0.038 & 0.020 & 0.012 & 0.010  & \textbf{ 0.520 } &  \sec{ 0.500} &   0.447 &   0.433 &   0.427 &   0.422 \\ 
Health & \textbf{0.153} & \sec{0.142}  & 0.022 & 0.084 & 0.011 & 0.015  &  \textbf{ 0.597 } &  \sec{ 0.549 } &   0.449 &   0.540 &   0.425 &   0.435 \\ 
Diaper & \textbf{0.134} & \sec{0.115}  & 0.023 & 0.018 & 0.013 & 0.012  &  \textbf{ 0.562 } &  \sec{ 0.546 } &   0.447 &   0.440 &   0.435 &   0.435 \\ 
Toys & \textbf{0.157} & \sec{0.150}  & 0.025 & 0.064 & 0.029 & 0.029  &     \textbf{ 0.591 } &  \sec{ 0.577 } &   0.446 &   0.472 &   0.448 &   0.449 \\ 
Bedding & \textbf{0.203} & \sec{0.191}  & 0.028 & 0.015 & 0.043 & 0.047  & \textbf{ 0.643 } &  \sec{ 0.623 } &   0.437 &   0.438 &   0.456 &   0.461 \\ 
Feeding & \textbf{0.100} & \sec{0.091}  & 0.026 & 0.023 & 0.020 & 0.019  & 
\textbf{ 0.550 } &  \sec{ 0.547 } &   0.459 &   0.453 &   0.454 &   0.452 \\ 
Apparel & \textbf{0.101} & \sec{0.093}  & 0.036 & 0.022 & 0.016 & 0.016  &  \textbf{ 0.558 } &  \sec{ 0.550 } &   0.459 &   0.452 &   0.446 &   0.444 \\ 
Media & \textbf{0.135} & \sec{0.130}  & 0.029 & 0.035 & 0.029 & 0.025  & \textbf{ 0.578 } &  \best{ 0.578 } &   0.474 &   0.470 &   0.461 &   0.461 \\
\hline\hline
\end{tabular}
\endgroup}
\caption{Prediction of subsets in product recommendation task. 
% Performance is measured in terms of Jaccard Coefficient (JC) and Mean NDCG@10 
%   for nine categories from Amazon baby registry records, for \our, Deep submodular function (DSF), mixture of submodular functions (SubMix),
% Facility location (FL), Determinantal point process (DPP) and Disparity Min (DisMin). 
% In all experiments, we use training, test, validation folds of equal size.
Numbers in \textbf{bold}   (\sec{underline})
indicate best (second best) performer. }
% %\vspace{-1mm}
\label{tab:real}
\end{table*}

\ehdr{Mean Jaccard coefficient (MJC)} Given a set $T$ in the test set, we first measure the overlap between $T$ and the first $|T|$ elements of $S'$  
using Jaccard coefficient, \ie, $JC(T) = |S'_{\le|T|}\cap T|/|S'_{\le|T|}\cup T|$ and then average over all subsets in the test set, \ie, $T\in \Scal_{\text{test}}$ to compute mean Jaccard coefficient~\cite{jaccard}. 

\ehdr{Mean NDCG@10} The greedy sampler outputs item sequence~$S'$. For each $T\in \Scal_{\text{test}}$, 
we compute the NDCG given by the order of first $10$ elements of $S'$, where $i\in S'$ is assigned a gold relevance label $1$, if
$i\in T$ and $0$, otherwise.
% $AP(T)=\frac{1}{|T|}\sum_{i\in [|S'\cup T|]} \II(i\in T) |S' _i \cap T|/i$, where $S' _i$ is the first $i$ elements in $S'$~\cite{schutze2008introduction}. 
Finally, we average over all test subsets 
to compute Mean NDCG@10.

\xhdr{Our method vs.\ baselines} We compare the subset selection ability of \our\ against several baselines which include two trainable submodular models: (i)~Deep submodular function (DSF) and
  (ii)~Mixture of submodular functions (SubMix) described in Section~\ref{sec:syn-setup};
two popular non-trainable submodular functions which include (iii)~Facility location (FL)~\cite{mirchandani1990discrete,frieze1974cost}, (iv)~Determinantal point process (DPP)~\cite{borodin2009determinantal}; and, a non-submodular function (v)~Disparity Min (DisMin) which is often used in data summarization~\cite{dasgupta2013summarization}. Here, we use the monotone submodular model of \our.  Appendix~\ref{app:real} contains more details about the baselines. We did not consider general purpose set functions, \eg, \settx, \deepset, etc., because they cannot be maximized using greedy-like algorithms and therefore, we cannot apply our proposed method in Section~\ref{sec:summarization}.  
 
\xhdr{Results}
\label{sec:results}
Table~\ref{tab:real} provides a comparative analysis across all candidate set functions, which shows that: 
%
%
%  \begin{wrapfigure}[7]{r}[-10pt]{0.22\textwidth}
% %\vspace{-3mm}
% \centering
% \includegraphics[width=0.24\textwidth]{FIG/efficiencybathjc.pdf}
% %\vspace{-4mm}
% \caption{Efficiency.}
% %\vspace{-1mm}
% \label{fig:efficiency}
% % \end{figure}
% % \end{minipage}
% \end{wrapfigure}
%
%
(i)~\our\ outperforms all the baselines across all the datasets; (ii)~DSF is the second best performer across all datasets;
and, (iii)~the performance of non-trainable set functions is poor,
as they are \emph{not trained to mimic} the set selection process.
% Recall our key motivation was to train parameters invariant to the  order of the elements of the input subset.
% \citet{tschiatschek2018differentiable} achieve this goal using a Monte Carlo average of the likelihood over many samples.   In Figure~\ref{fig:efficiency}, we compare our method with theirs for a representative dataset (Gear),  which shows that our method is ${\ge}4\times$ faster.

% \input{099end}
% %\vspace{-3mm}
\section{Conclusion}
\label{sec:end}
% %\vspace{-2mm}

We introduced \our: a family of submodular functions, represented by neural networks that implement quadrature-based numerical integration, and supports end-to-end backpropagation through these integration operators. We designed a permutation adversarial subset selection method, which ensures that the estimated parameters are independent of greedy item selection order.  On both synthetic and real datasets, \our{} improves upon recent competitive formulations. 
Our work opens up several avenues of future work.  One can extend our work for $\gamma$-weakly submodular functions~\cite{elenberg2016restricted}. 
Another  extension of our work is to leverage other connections to convexity, \eg, the Lovasz extension~\cite{lovas0,bach2} similar to~\citet{karalias2021neural}.

% Applications may not always provide supervision in the form of (set, value) but instead only provide high-value sets, as a form of distant supervision.  We extend \our{} to address this realistic scenario.  On both synthetic and real set function estimation and item set selection problems, \our{} improves upon recent competitive formulations.

\newpage
%\bibliography{refs}
%\bibliographystyle{plainnat}
\bibliography{refs}
\bibliographystyle{plainnat}

\newpage
\onecolumn
\appendix

\begin{center}
\Large\bfseries\ztitle \\ (Appendix)\\
% \normalsize \ad{New contents added during rebuttal are colored blue}.
\end{center}

 \section{Potential limitations of our work}
\label{app:our-lim}

One of the key limitations of our work is that our neural models are limited to modeling concave composed modular functions. However the class of submodular functions are larger. One of the way to address this problem to model a submodular function using Lovasz extension~\cite{lovas0,lovas1}. Another key limitation is our approach cannot model weakly submodular functions at large, which is superset of approximate submodular functions modeled here. We would like to extend our work in this context. Our differentiable subset selection method does not have access to supervision of set values. Thus training only from high value subset can lead to high bias towards a specific set of elements. An interesting direction would be to mitigate such bias.

\section{Proofs of the technical results for Section~\ref{sec:subm}}
\label{app:sec:alpha}

\subsection{Formal justification of the submodularity of \our\ given by Eq.~\eqref{eq:sub-model}}
% \label{app:sec:alpha}

\begin{proposition} 
Let  $m^{(n)} _{\theta} : 2^V \to \RR^+ $ be a modular function, i.e.,  $m^{(n)}  _{\theta} (S) = \sum_{s\in S}m^{(n)} _{\theta}(\{s\})$;   $\phi_{\theta}$ be a monotone concave function. Then, $F_{\theta}(S)$ computed using~Eq.~\eqref{eq:sub-model} is  monotone submodular.
\end{proposition}

%\vspace{-3mm}
\begin{proof}
 We proof this by induction. Clearly,  $F^{(0)}$ is monotone submodular. Now, assume that $F^{(n-1)}(S)$ is monotone submodular.
 Then,  $R(S) = \lambda \cdot F^{(n-1)} _{\theta}(S)  + (1-\lambda)\cdot m_{\theta} ^{(n)} (S)$ is monotone submodular. Hence, from Proposition~\ref{prop-basic} (i), we have $F^{(n)} _{\theta}(S) = \phi_{\theta} (R(S))$ to be submodular.
%  To prove $R(S)$ is \asb, we proceed as follows: 
%  \begin{align}
%   R(s\given S) & = \lambda \cdot F^{(n-1)}(s\given S)  + (1-\lambda)\cdot m_{\theta} ^{(n)} ({s})\nn\\
%   & \ge \lambda \alpha \cdot F^{(n-1)}(s\given T)  + \alpha (1-\lambda)\cdot m_{\theta} ^{(n)} ({s}) \quad \text{( Since, $F(S)$ is \asb\ and $\alpha  m_{\theta} ^{(n)} ({s}) <  m_{\theta} ^{(n)} ({s})$)}\nn\\
%  & \ge R(s\given T).
%  \end{align}
\end{proof}
\subsection{Proof of Theorem~\ref{thm:diffc}}

\begin{numtheorem}{\ref{thm:diffc}}
 Given the functions $\diffc:\RR~\to~\RR^+$ and $m:V~\to~[0,1]$. Then, the set function $F(S) =  \diffc(\sum_{s\in S}m(\set{s}))$  is monotone 
 $\alpha$-submodular for  $|S|\le k$, if $\diffc(x)$ is increasing in $x$ and
 \begin{align}
 \frac{\p ^2 \diffc(x)}{\p x^2} \le \frac{1}{k}\log\left(\frac{1}{\alpha}\right)\frac{\p  \diffc(x)}{\p x} 
\end{align}
 \end{numtheorem}
\begin{proof} Since, both $\diffc$ and $F$ is monotone, $\diffc\circ F$ is monotone. Shifting $x$ to $x+y$ with $y>0$, we have:
 \begin{align}
  \frac{\p ^2 \diffc(x+y)}{\p y^2} \le \ga \frac{\p  \diffc(x+y)}{\p y}\label{eq:diffc-inter} 
\end{align}
where, $\ga = \frac{1}{k}\log\left(\frac{1}{\alpha}\right)$. Eq.~\eqref{eq:diffc-inter}
implies that
\begin{align}
&  e^{-y\ga} \frac{\p ^2 \diffc(x+y)}{\p y^2}-  \ga e^{-y\ga}\frac{\p  \diffc(x+y)}{\p y} \le 0  \implies \frac{\p}{\p y}\left( \frac{ e^{-y\ga} \p  \diffc(x+y) }  {\p y} \right) \le 0
\end{align}
Hence, $\displaystyle \frac{ e^{-y\ga} \p  \diffc(x+y) }  {\p y}$ is decreasing function in $y$. Hence, 
\begin{align}
 \frac{ e^{-y\ga} \p  \diffc(x+y) }  {\p y} \le \left. \frac{ e^{-y\ga} \p  \diffc(x+y) }  {\p y} \right| _{y=0} =  \frac{  \p  \diffc(x) }  {\p x} \label{eq:inter}
\end{align}
Next, we define $\diffc_S(\bullet) = \diffc(\bullet+m(S))$ for all $S$ and then we compute 
\begin{align}
\frac{ \diffc(m(S\cup s))- \diffc(m(S))}{ \diffc(m(T\cup s))- \diffc(m(T))}& =  \frac{ \diffc(m({s}) + m(S))- \diffc(m(S))}{ \diffc(m({s}) + m(T))- \diffc(m(T))} \qquad \text{(Since $m$ is modular)} \\
& = \frac{ \diffc_S(m({s}))- \diffc_S(0)}{ \diffc_T(m({s}))- \diffc_T(0)} \\
& =  \dfrac{\left.\dfrac{  \p  \diffc_S(x) }  {\p x} \right|_{x=c}}{\left.\dfrac{  \p  \diffc_T(x) }  {\p x}\right|_{x=c}}\qquad\text{(for some $c\in(0, m({s}))$; Cauchy's mean value Theorem)}  \nn\\
& \ge \exp(-\ga [m(T)-m(S)]) \qquad\text{(Using Eq.~\eqref{eq:inter})}\nn\\
& \ge \exp(-\log(1/\alpha)) \qquad \text{(Since $|S|, |T|\le k$)}
\end{align}
\end{proof}

\subsection{Proof of proposition~\ref{prop:alpha1}}

\begin{numproposition}{\ref{prop:alpha1}}
Given a monotone $\alpha$-submodular function $F(\bullet)$. Then, $\phi(F(S))$ is monotone $\alpha$-submodular, if $\phi(\bullet)$  is an increasing concave function.
\end{numproposition}

\begin{proof}  
Assume $S\subset T$ and $s\in V\cp T$. Since $F(\bullet)$ is monotone, we have $F(S)\le F(T)$ and $F(S\cup s)\le F(T\cup s)$. Using concavity of $\diffc$ and comparing the slope of two chords, we have:
 \begin{align}
&  \frac{\diffc(F(S\cup s))-\diffc(F(S))}{F(S\cup s)-F(S)}  \ge   \frac{\diffc(F(T\cup s))-\diffc(F(T))}{F(T\cup s)-F(T)} \nn\\
\implies &\diffc(F(S\cup s))-\diffc(F(S)) \ge  \frac{F(S\cup s)-F(S)}{F(T\cup s)-F(T)} \left[\diffc(F(T\cup s))-\diffc(F(T))\right] \label{eq:chord}
 \end{align}
$F$ is monotone. Hence, the last inequality is obtained by multiply both sides by the positive quantity ${F(S\cup s)-F(S)}$ which keeps the sign of the inequality unchanged.

Now, since $\diffc$ is an increasing function and $F$ is monotone, $\diffc(F(T\cup s))-\diffc(F(T)) \ge 0$. Then, since $F$ is monotone \asb, we have $ \frac{F(S\cup s)-F(S)}{F(T\cup s)-F(T)} \ge {\alpha}$. Hence, we have $\frac{F(S\cup s)-F(S)}{F(T\cup s)-F(T)} \left[\diffc(F(T\cup s))-\diffc(F(T)) \right]\ge  {\alpha}[\diffc(F(T\cup s))-\diffc(F(T))]$. Together with Eq.~\eqref{eq:chord}, it finally gives
$$ \diffc(F(S\cup s))-\diffc(F(S))  \ge \frac{F(S\cup s)-F(S)}{F(T\cup s)-F(T)} \left[\diffc(F(T\cup s))-\diffc(F(T)) \right]\ge {\alpha}[  \diffc(F(T\cup s))-\diffc(F(T))].$$
% Moreover $F$ is increasing and $\diffc$ is increasing. Hence $\diffc \circ F$ is increasing.
\end{proof}
 
\subsection{Proof of Proposition~\ref{prop:ua}}
 \begin{numproposition}{\ref{prop:ua}}
Given an universal set $V$, a constant $\epsilon>0$ and a submodular function $F(S) = \phi(\sum_{s\in S} m(\zb_s))$  where   $\zb_s \in \RR^d$, $S\subset V$, $0 \le m(\zb) < \infty$ for all $\zb \in \RR^d$. Note that here, $F$ is normalized and therefore, $F(\emptyset)=\phi(0)=0$. Then there exists two fully connected neural networks   $m_{\theta_1}$  and $h_{\theta_2} $   of width $d+4$ and $5$ respectively, each with ReLU activation function, such that the following conditions hold:
\begin{align}
   \left|F(S) -   \int_{a=0} ^{a=\sum_{s\in S}m_{\theta_1}(\zb_s)} \int_{b=a} ^{b=\infty} h_{\theta_2}(b) \, \d b \, \d a. \right|    \le \epsilon \quad \forall \ S\subset V\label{eq:ua}
\end{align}
\end{numproposition}
 
\begin{proof}
Since our functions are normalized, we have $F(\emptyset) = 0=\phi(0)$.  Assume $d\phi(b)/db \to 0$ as $b\to \infty$. 
 Note that the condition that $\lim_{b\to\infty} d\phi(b)/db \to 0$ is not a restriction since the maximum value of $x$ that goes as input to provide output $\phi(x)$ is $x_{\max} = \sum_{s\in V} m (\zb_s)$ which is finite. Therefore, one can always define $\phi(\bullet)$ outside that regime ($x>x_{\max}$) as constant zero. Then we can write:
\begin{align}
    \phi(x) = \int_{a=0} ^{a=x} \int_{b=a} ^{b=\infty} \left[-\frac{d^2\phi (b)}{db^2}\right] \, \d b \, \d a   \label{eq:phi=phi1}
\end{align}
 
To prove the above, one can define the RHS of Eq.~\eqref{eq:phi=phi1} as say, $\phi_1(x)$.
We note that $\phi(0)=0$ because it is normalized. Thus,
we have:
\begin{align}
    &\phi_1(0)=\phi(0) \\
    &\frac{d\phi_1(x)}{dx}=\frac{d\phi(x)}{dx} \quad \text{(Since, $\lim_{b\to\infty} d\phi(b)/db \to 0$)} 
\end{align}
These two conditions give us:
$\phi(x) = \phi_1(x)$. Now, we can write $F(S)$ as follows:
% $\phi_1(0)=\phi(0)$. Moreover, since  $\lim_{b\to\infty} d\phi(b)/db \to 0$, we have $\frac{d\phi(x)}{dx}=\frac{d\phi_1(x)}{dx}$. 
 \begin{align}
 F(S) & =   \int_{a=0} ^{a=\sum_{s\in S}m(\zb_s)} \int_{b=a} ^{b=\infty} \left[-\frac{d^2\phi (b)}{db^2}\right] \, \d b \, \d a\end{align}
Let us define $h(b) =-\frac{d^2\phi (b)}{db^2} $. Choose $\epsilon_m >0$ and $\epsilon_h>0$. According to~\cite[Theorem 1]{lu2017expressive}, we can say that it is possible to find ReLU neural networks $m_{\theta_1}$ and $h_{\theta_2}$ for widths $d+4$ and $5$ respectively, for which
we have, 
\begin{align}
&    \int_{\RR^d}|m(\zb)-m_{\theta_1}(\zb)| d\zb < \epsilon_m\\
& \int_{\RR^+ }|h(b)-h_{\theta_2}(b)| \d b < \epsilon_h
\end{align}
For continuous functions $m$,  the first condition suggest that there exists $\epsilon' _m$ for which $|m(\zb)-m_{\theta_1}(\zb)| \le \epsilon' _m$. 

Then we have the following:
\begin{align}
  & F(S) -   \int_{a=0} ^{a=\sum_{s\in S}m_{\theta_1}(\zb_s)} \int_{b=a} ^{b=\infty}  h_{\theta_2}(b)\, \d b \, \d a \\
  &=  \int_{a=0} ^{a=\sum_{s\in S}m(\zb_s)} \int_{b=a} ^{b=\infty}  h(b)\, \d b \, \d a -   \int_{a=0} ^{a=\sum_{s\in S}m_{\theta_1}(\zb_s)} \int_{b=a} ^{b=\infty}  h (b)\, \d b \, \d a \\
  &\quad +   \int_{a=0} ^{a=\sum_{s\in S}m_{\theta_1}(\zb_s)} \int_{b=a} ^{b=\infty}  h (b)\, \d b \, \d a- \int_{a=0} ^{a=\sum_{s\in S}m_{\theta_1}(\zb_s)} \int_{b=a} ^{b=\infty}  h_{\theta_2}(b)\, \d b \, \d a
  \end{align}
  \begin{align}
  & = \int_{a=\sum_{s\in S} m_{\theta_1}(\zb_s)} ^{a=\sum_{s\in S} m(\zb_s)} \int_{b=a} ^{b=\infty}  h(b)\, \d b \, \d a + \int_{a=\sum_{s\in S} m (\zb_s)} ^{a=\sum_{s\in S} m_{\theta_1}(\zb_s)} \int_{b=a} ^{b=\infty}  [h(b) -h_{\theta_2}(b)]\, \d b \, \d a\\
  & \qquad \qquad + \int_{a=\sum_{s\in S} m(\zb_s)} ^{0} \int_{b=a} ^{b=\infty}  [h(b) -h_{\theta_2}(b)]\, \d b \, \d a
\end{align}
This gives us:
\begin{align}
    &\left|F(S) -   \int_{a=0} ^{a=\sum_{s\in S}m_{\theta_1}(\zb_s)} \int_{b=a} ^{b=\infty}  h_{\theta_2}(b)\, \d b \, \d a\right|\\
    &\quad \le |V|\epsilon' _m \int_{b=0} ^{b=\infty}  h(b)\, \d b \, \d a +|V|\epsilon' _m \epsilon_h + +|V|m_{\max} \epsilon_h. 
\end{align}
Here, $m_{\max} = \max_{s\in V} m(\zb_s)$. One can choose $\epsilon_{\bullet}$ in order to set the RHS to $\epsilon$.
 \end{proof} 
% \begin{numproposition}{\ref{prop:neural-alpha}}
% Let  $m^{(n)} _{\theta} : 2^V \to \RR^+ $ be a modular function, i.e.,  $m^{(n)}  _{\theta} (S) = \sum_{s\in S}m^{(n)} _{\theta}(\{s\})$; $\diffc_{\theta}(\bullet)$ satisfy the conditions of Theorem~\ref{thm:diffc} with $\alpha\in(0,1)$; $\phi_{\theta}$ be a monotone concave function. Then, $F_{\theta}(S)$ is a monotone \asb\ function. 
% \end{numproposition}

\subsection{Formal result showing that $F_{\theta}$ computed using Eq.~\eqref{eq:alpha-sub-model}  is $\alpha$-submodular}

\begin{proposition}
\label{prop:neural-alpha}
Let  $m^{(n)} _{\theta} : 2^V \to \RR^+ $ be a modular function, i.e.,  $m^{(n)}  _{\theta} (S) = \sum_{s\in S}m^{(n)} _{\theta}(\{s\})$; $\diffc_{\theta}(\bullet)$ satisfy the conditions of Theorem~\ref{thm:diffc} with $\alpha\in(0,1)$; $\phi_{\theta}$ be a monotone concave function. Then, $F_{\theta}(S)$ computed using Eq.~\eqref{eq:alpha-sub-model} is a monotone \asb\ function. 
\end{proposition}

\ehdr{Proof} We proof this by induction. From  Theorem~\ref{thm:diffc}, $F^{(0)}$ is monotone \asb. Now, assume that $F^{(n-1)}(S)$ is monotone \asb.
 Then, we see that $R(S) = \lambda \cdot F^{(n-1)}(S)  + (1-\lambda)\cdot m_{\theta} ^{(n)} (S)$ is monotone. To prove $R(S)$ is \asb, we proceed as follows: 
 \begin{align}
  R(s\given S) & = \lambda \cdot F^{(n-1)}(s\given S)  + (1-\lambda)\cdot m_{\theta} ^{(n)} ({s})\nn\\
  & \ge \lambda \alpha \cdot F^{(n-1)}(s\given T)  + \alpha (1-\lambda)\cdot m_{\theta} ^{(n)} ({s}) \  \text{ ($F(S)$ is \asb, $\alpha  m_{\theta} ^{(n)} ({s}) <  m_{\theta} ^{(n)} ({s})$)}\nn\\
 & \ge \alpha R(s\given T).
 \end{align}
Applying Proposition~\ref{prop:alpha1} on the above leads to the final result. 

%\vspace{-1mm}
\section{Decoupling integrals for \asb\ and non-monotone submodular functions}
\label{app:decoupling}
\emph{--- Monotone \asb\ function.} In case of monotone \asb\ functions, we first break the double integral of $\phi_{\theta}$ used in the recursion~\eqref{eq:alpha-sub-model} using the two single integrals described in Eq.~\eqref{eq:decouple-mon-sub}. In addition, we  
need to model $\diffc_{\theta}$ in Eq.~\eqref{eq:phiam} as follows:
\begin{align}
%\vspace{-2mm}
   \hspace{-3mm} \diffc_{\theta}(x) & = \int_{0} ^{x} \diffc' _{\theta}(a) \, \d a,  \ \
     \diffc' _{\theta}(x) = \int_{x} ^{\infty} \hspace{-3mm} e^{a\ga} g_{\gamma}(a) \, \d a \label{eq:decouple-mon-sub-phi}
\end{align}
Then, we apply two regularizers with the loss~\eqref{eq:loss}: one for $\phi' _{\theta}$ similar to Eq.~\eqref{eq:loss-mon-decoup} and the other for the second integral equation in Eq.~\eqref{eq:decouple-mon-sub-phi} which is computed as
$\sum_{i \in [I]}\rho\left(\diffc' _{\theta}(m_{\theta} ^{(0)} (S_i)) -\int_{m_{\theta} ^{(0)}(S_i)} ^{\infty}  e^{a\ga} g_{\gamma}(a) \, \d a \right)^2$.
Then, we minimize the regularized loss with respect to $\theta,\beta, \gamma$. 
 
\ehdr{--- Non-monotone submodular function} In case of  non-monotone submodular function models, we decouple $\psi_{\theta}$ in Eq.~\eqref{eq:psim} into the following set of intergals:
\begin{align}
%\vspace{-1mm}
&\hspace{-3mm} \psi_{\theta}(x)  = \int_{0} ^{x} \psi' _{h,\theta}(a) \d a - \int_{x_{\max}-x} ^{x_{\max}} \psi' _{g,\theta}(a)  \, \d a \\ 
&\hspace{-3mm} \psi' _{h,\theta}(x)  = \int_{x} ^{\infty} h_{\beta}(b), \d b,\quad 
\psi' _{g,\theta}(x)  = \int_{x} ^{\infty} g_{\gamma}(b) \, \d b \label{eq:non-mon-integral} \\[-3ex]\nn
%      \psi_{\theta}(x)&  = \int_{a=0} ^{a=x} \int_{b=a} ^{b=\infty} \gt(b) \, \d b \, \d a \nn \\
% &\qquad\quad- \int_{a=x_{\max}-x} ^{a=x_{\max}} \int_{b=0} ^{b=a}  \ggt(b) \, \d b \, \d a, \label{eq:psim}
\end{align}
Then, we add the regularizers corresponding to the last two integral equations~\eqref{eq:non-mon-integral}
% \begin{align}
% & \sum_{i\in[I]}\rho(\psi' _{h,\theta}\bigg(m_{\theta}(S_i))  - \int_{m_{\theta}(S_i)} ^{\infty}  h_{\beta}(b), \d b \bigg)^2, \text{  and, }\nn\\ 
% &  \sum_{i\in[I]}\rho(\psi' _{g,\theta}\bigg(m_{\theta}(S_i))  - \int_{m_{\theta}(S_i)} ^{\infty} g_{\gamma}(b), \d b\bigg)^2
% \nn
%  \end{align}
to the loss~\eqref{eq:loss} and then minimize it to train $\theta, \beta, \phi$.

\section{Additional  details about experiments on learning from (\tset, \tvalue) pairs}
\label{app:syn}
\subsection{Additional details about the data generation}
As mentioned in Section~\ref{sec:syn-setup}, we generate $|V|=10^4$ items. We draw the feature vector $\feat_\els$ for each item $\els\in V$ uniformly at random \ie,
$\feat_\els\in \text{Unif}[0,1]^{d}$, where $d=10$. Here, we use such a generative process for the features since our synthetic set functions often require positive input. Then, we generate subsets $S$ of different sizes by gathering elements from the universal set $V$ as follows. We randomly shuffle the elements of $V$ to obtain a sequence of elements $\set{s_1,...,s_{|V|}}$. 
We construct $|V|$ subsets by gathering top $j$ elements for $j=1,..,|V|$, \ie,
$\Scal =\set{S} = \set{\set{s_1,...,s_j}\,|\, j\le |V|}$.  
\todo{AD here}

\subsection{Implementation details of \our}

We model the integrands $h_{\theta}$ of the submodular funcions using one input layer, three hidden layers and one output layer, each with width 50. Here, the input and hidden layers are built using one
linear and ReLU units and the output layer is an ELU activation unit. For forward and backward passes under integration, we adapt the code provided by~\citet{umnn}
in our setup. Moreover we choose the modular function $m_{\theta} ^{(n)}(\set{s}) = \theta_m ^\top \zb_s$.
We set the value of maximum number of steps using cross validation, which gives $N=2$ for all datasets. 
We set the value of weight decay as $1e-4$ and learning rate as $2e-3$.

\subsection{Implementation details of the baselines}

\xhdr{\settx~\cite{lee2019set}} Our implementation of \settx\ has one input layer, two hidden layers and one output layer. We would like to hightlight that with increased number of parameters, \settx\ consumed large GPU memory (max GPU usage > 8GB for even 29 parameters). This is because of two reasons: (1) the set transformer performs all-to-all attention architecture and (2) the size of universe in our experiments is $|V|=10000$. \settx\ involves several concatenation operation which blows up the intermediate tensors.  

\xhdr{Deep set~\cite{zaheer2017deep}} Our implementation of deep set model is  similar to the integrand of \our, \ie, it consists of one input layer, three hidden layers and one output layer, each with width 50. Here, the input and hidden layers is built of one linear and ReLU units whereas, the output layer is an ELU activation unit. Here, we set the value of weight decay as $10^{-4}$ and learning rate as $2{\times}10^{-3}$.  

\xhdr{Deep submodular function (DSF)~\cite{bilmes2017deep}}
DSF makes no prescription about the choice of network depth, width, or concave functions.
Other researchers commonly make simple choices such as fully-connected layers with arbitrary concave activation \citep{pmlr-v151-manupriya22a}.  Similarly, we use the monotone concave function $\phi_{\theta}(x) = \log(x +\theta)$ \todo{how avoid NaN? also, $\theta$ this $\theta$ that} with trained $\theta\in\RR_+$ and the modular function $m_{\theta}(s) = \theta_m^\top \zb_s$. 
Similar to our method, we use the network depth $N=2$ for DSF.
In our experiments, we found that not using the offset gave unacceptably poor predictive performance. We set the value of weight decay as $10^{-4}$ and learning rate as $10^{-3}$.  

\xhdr{Mixture submodular function (SubMix)~\cite{tschiatschek2014learning}}  Here, we consider $F_{\theta}(S) = \theta_1 \log(\sum_{s\in S}\theta_m ^\top \zb_s) + \theta_2 \log \log(\sum_{s\in S} \theta_m ^\top \zb_s) + \theta_3  \log \log \log(\sum_{s\in S} \theta_m ^\top \zb_s)$. Note that this design does not make sure 
$ \log(\sum_{s\in S} \theta_m ^\top \zb_s)>0$ or $ \log \log(\sum_{s\in S} \theta_m ^\top \zb_s)>0$. However, with valid initial conditions $\theta_{m,0}$ where $\log(\sum_{s\in S} \theta_{m,0} ^\top \zb_s)>0$ and
$ \log \log(\sum_{s\in S} \theta_{m,0} ^\top \zb_s)>0$ and the current learning rate $10^{-3}$, we observed that $\theta_m$ always ensured that
$\log(\sum_{s\in S} \theta_m ^\top \zb_s)>0$ and
$ \log \log(\sum_{s\in S} \theta_m ^\top \zb_s)>0$ throughout our training. 
 
Initially we started with 
$F_{\theta}(S) = \theta_1 \log(1+\sum_{s\in S}\theta_m ^\top \zb_s) + \theta_2 \log (1+\log(1+\sum_{s\in S} \theta_m ^\top \zb_s)) + \theta_3 \log(1+( \log (1+\log (1+\sum_{s\in S} \theta_m ^\top \zb_s))))$, which always would ensure that $\log(\sum_{s\in S} \theta_m ^\top \zb_s)>0$ and
$ \log \log(\sum_{s\in S} \theta_m ^\top \zb_s)>0$. But we observed that the performance deteriorates than the current candidate which does not add $1$ to each log term.

Moreover, we observed that adding additional component did not improve accuracy.
Here, we set the value of weight decay as $10^{-4}$ and learning rate as $10^{-3}$.  

In all models, we set the initial value of the  parameter vector of the modular function to be $\theta_m = \bm{1}$ which ensured that the final trained model $\theta_m \ge 0$. In all experiments, we set the batch size as $66$. For each model, we train for 400 epochs and choose the training model which shows the best  performance in last 10 epochs. We choose the best initial model based on the performance of final trained model on the validation set.

 \subsection{Computation of $\alpha$ in synthetically planted functions}
We define the curvature $F$~\cite{vondraksubmodularity} as: 
\begin{align}
    \text{curv}_F =  1-\min_{S, j\not \in S}\frac{F(j\given S)}{F(j\given \emptyset)}
\end{align}
Define  $ z_{\max} = {\max_{s\in V} ||\pmb{z}||_{\infty}}, $ $z_\text{min} =   \min_{s\in V, i \in [d]}$ and assume  $z_{\min} > e^{-2}$.  

\xhdr{Log $\times$ LogDet}
First we consider $F(S) = [\log (\sum_{\els\in S} \bm{1}^\top \feat_\els)]\cdot[\log \text{det}(\II+\sum_{\els\in S}\feat_\els \feat_\els ^\top)]$. Assume that
\begin{align}
    &f(S) = \log (\sum_{\els\in S} \bm{1}^\top \feat_\els) \nn\\
    &g(S) = \log \text{det}(\II+\sum_{\els\in S}\feat_\els \feat_\els ^\top)
\end{align}
Now, we have: 
\begin{align}
F(S\cup k) - F(S) & = f(S\cup k)\, g(S\cup k)-f(S)\, g(S)\nn\\
& =  f(S\cup k) (g(S\cup k)- g(S)) + g(S) (f(S\cup k)- f(S))\nn \\ 
& \ge f_{\min} \, \text{curv}_{g} \, g_{\min} 
\end{align}
Similarly,
\begin{align}
F(T\cup k) - F(T) & = f(T\cup k)\, g(T\cup k)-f(T)\, g(T)\nn\\
& =  f(T\cup k) (g(T\cup k)- g(T)) + g(T) (f(T\cup k)- f(T)) \label{eq:diff-f}
\end{align}
Now, $\log \det (A) \le \tr(A-\II)$. Hence, we have:
\begin{align}
    g(T\cup k)- g(T) & = \log \det\left( \II + \left( \II + \sum_{s\in T}\feat_\els \feat_\els ^\top \right)^{-1} \zb_k \zb_k ^{\top} \right) \nn\\
    & \le \zb_k ^{\top} \left( \II + \sum_{s\in T} \feat_\els \feat_\els ^\top \right)^{-1} \zb_k 
\end{align}
Moreover, $f(T \cup k) = \log (\sum_{\els\in T \cup k} \bm{1}^\top \feat_\els) \le \sum_{\els\in T \cup k} \bm{1}^\top \feat_\els $. Hence, $f(T\cup k) (g(T\cup k)- g(T))$ satisfies:
\begin{align}
    f(T\cup k) (g(T\cup k)- g(T)) \le  \zb_k ^{\top} \left( \frac{1}{\sum_{\els\in T \cup k} \bm{1}^\top \feat_\els}\left(\II + \sum_{s\in T} \feat_\els \feat_\els ^\top \right)\right)^{-1} \zb_k . 
\end{align}
Here, we make some crude probabilistic  argument.
Since $\feat_\els$ is iid uniform random variables, 
$\sum_{s\in T}  \feat_\els \feat_\els ^\top  \approx |T|\left[\II/12 + \bm{1}\bm{1} ^T/4\right] $ and 
\begin{align}
f(T\cup k) (g(T\cup k)- g(T)) \le {24d^3 z^3 _{\max}}   
\end{align} 
The second term in Eq.~\eqref{eq:diff-f} shows that:
\begin{align}
    g(T) (f(T\cup k)- f(T)) & \le  \log \text{det}\left(\II+\sum_{\els\in T}\feat_\els \feat_\els ^\top\right) \cdot \log\left(1 + \frac{\bm{1}^{\top}\zb_k}  {\sum_{s\in T} \bm{1}^{\top}\zb_s}\right) \nn\\
    & \le \tr\left( \sum_{s\in T}\feat_\els \feat_\els ^\top\right)\cdot \frac{\bm{1}^{\top}\zb_k}  {\sum_{s\in T} \bm{1}^{\top}\zb_s} \le d^2 z_{\max} ^2 / z_{\min} ^2
\end{align}
 Hence, $F(S)$ is $\alpha$-submodular with
 $$\alpha> \alpha^* =  \frac{f_{\min} \, \max\{\text{curv}_{g},\text{curv}_{f}\} \, g_{\min}}{ d^2 z_{\max} ^2 / z_{\min} ^2 + 24d^3 z^3 _{\max} }    $$
% Now, for pd matrix $A^{-1} < \sum_{i\in[d]} \frac{1}{\text{Eigen}_i (A)}\le d^2 \sum_{i\in[d]} \text{Eigen}_i (A) = d^2 \text{Trace} (A)$. Using this result,  
% \newpage

\xhdr{Log $\times$ Sqrt}Now consider $F(S) = \sum_{s\in S} \log\left(\bm{1}^\top \zb_s\right) \sqrt{\bm{1}^\top \zb_s}$. By mean value theorem:
\begin{align}
    F(S\cup k)-F(S) & =  (\bm{1}^{\top} \zb_k) \frac{d}{dx} \log x   \sqrt{x} \bigg|_{x \in (\sum_{s\in S}\bm{1}^\top \zb_s, \sum_{s\in S\cup k}\bm{1}^\top \zb_s)} \nn\\
    & = (\bm{1}^{\top} \zb_k) \frac{2+\log x}{2\sqrt{x}}
\end{align}
Similarly ${F(T\cup k)-F(T)} = (\bm{1}^{\top} \zb_k)  \max_{y} \frac{2+\log y}{2\sqrt{y}} $ where $y   \in (\sum_{s\in T}\bm{1}^\top \zb_s, \sum_{s\in T\cup k}\bm{1}^\top \zb_s)$.
\begin{align}
    \alpha \ge  \frac{\frac{2+\log x}{2\sqrt{x}}}{ \frac{2+\log y}{2\sqrt{y}}   } \ge \frac{2+\log z_{\min}}{2\sqrt{z_{\min}}}
\end{align}
The above is due to the fact that: $\max_{y} \frac{2+\log y}{2\sqrt{y}} = 1$ at $y=1$.

\section{{Additional experiments with synthetic data}}

% \begin{wraptable}[4]{r}{0.5\textwidth}
% %\vspace{-4mm}
\begin{table}[h]
\centering
\small
% \adjustbox{max width=0.5\textwidth}{ \tabcolsep 3pt
% \begingroup \footnotesize
\begin{tabular}{c||c|c|c}
\hline 
& Log    & LogDet     & FL    \\ \hline        
Decoupling & \best{0.015 $\pm$ 0.000} & \best{0.013 $\pm$ 0.000} & \best{0.022 $\pm$ 0.000} \\
End-end &{0.078 $\pm$ 0.001} & {0.073 $\pm$ 0.001} & {0.078 $\pm$ 0.001} \\
Our ($N=1$) &{0.089 $\pm$ 0.001} & {0.075 $\pm$ 0.001} & {0.089 $\pm$ 0.001}
\\ \hline 
\end{tabular}
% \endgroup}
% %\vspace{-2mm}
\caption{Variants of our approach.
% We observe that decoupling into independent integrals improves the predictive performance. 
}
% %\vspace{-1mm}
\label{tab:training}
\end{table}
\begin{table*}[h]
\centering
\small
    \begin{tabular}{c|c|c|c|c}
    \hline
        \our & \settx & \deepset & DSF & SubMix \\ \hline
        0.055 & 0.119 & 0.160 & 2.31 & 1.770 \\ \hline
    \end{tabular}
\caption{Performance measured in terms of RMSE on synthetically generated examples using $F(S) = \min(\sum_{s\in S} \pmb{1}^{\top} \pmb{z}_s,  b+\min(r , \sum_{s\in S}\pmb{1}^{\top} \pmb{z}_s),a)$
We set $r = \sum_{s\in V} \pmb{1}^{\top} \pmb{z}_s / 3$ ,  $b =  \sum_{s\in V} \pmb{1}^{\top} \pmb{z}_s / 6$, $a =   \sum_{s\in V} \pmb{1}^{\top} \pmb{z}_s / 2$. We observe that our model significantly outperforms the baselines.
}\label{tab:reb}
\end{table*}

% \end{wraptable}
\xhdr{Ablation study} We compare different variants of our approach: (i)~\our, trained by decoupling the double integral into independent integrals (Eq.~\eqref{eq:loss-mon-decoup}); (ii)~\our, trained using  end-to-end training via backpropagation  through double integrals (Section~\ref{sec:sec:train}); and 
(iii)~\our\ with $N=1$, where $N$ is the number of steps of the recursions~\eqref{eq:sub-model}. Table~\ref{tab:training} summarizes the results which reveal the following observations. (1)~We achieve substantial performance gain via decoupling into independent integrals. Although decoupling integrals is an approximation of end-to-end training, it also provides a more smooth loss surface than the loss on the double integral. (2)~Running \our{} for only a single step significantly deteriorates  performance.

\xhdr{Performance with additional planted synthetic function}
Here, we consider a featurized form of the function used in the proof of the lower bound~\cite{reb}:

$F(S) = \min(\sum_{s\in S} \pmb{1}^{\top} \pmb{z}_s,  b+\min(r , \sum_{s\in S}\pmb{1}^{\top} \pmb{z}_s),a)$.
We set $r = \sum_{s\in V} \pmb{1}^{\top} \pmb{z}_s / 3$ ,  $b =  \sum_{s\in V} \pmb{1}^{\top} \pmb{z}_s / 6$, $a =   \sum_{s\in V} \pmb{1}^{\top} \pmb{z}_s / 2$.

Table~\ref{tab:reb} summarizes the results in terms of RMSE. We observe that our method outperforms other methods by a substantial margin. 

\xhdr{Scalability analysis}
We report per-epoch training time of different methods in the context of training by (set, value) pairs.

\begin{table}[h]
    \centering
    \begin{tabular}{ c|c|c|c|c }
    \hline
        \our & \deepset & \settx & DSF & SubMix \\ \hline
        7.2649 & 1.3009 & 2.2856 & 1.39 & 1.329 \\ \hline
    \end{tabular}
    \caption{Per epoch time (in second) for different methods.}
\end{table}

While our method is slower than the baseline methods (mainly due to the numerical integration), it offers significantly higher accuracy than other methods.

\xhdr{Variation of the performance of \settx\ against the number of parameters} \settx\ is an excellent neural set function with very high expressive power. This expressive power comes from its ability to incorporate interaction between elements. However, it consumes significantly large memory in practice.  Our method and all the other baselines do not incorporate interactions and thus consume lower memory. Thus, we reduced the number of parameters of \settx\ so that it consumes similar memory as our method (8--10 GB) for a fair comparison. 
Here, we experiment with different batch size (B) and increased number (P) of parameters  (as GPU memory permitted).  Results are as follows.

\begin{table}[h]
    \centering
    \begin{tabular}{l|l|l|l|l}
    \hline
        P=29, B=66 & P=139, B=40 & P=321, B=40 & P=321, B=17 & P=439, B=17 \\ \hline
        0.063 & 0.069 & 0.056 & 0.054 & 0.055 \\ \hline
    \end{tabular}
    \caption{RMSE for different configurations of Set transformer for Facility Location dataset}
\end{table}

As expected, the performance indeed improves if we increase the number of parameters. However, even with significantly small batch size, an increased number of parameters (otherwise accuracy drops) led to large computation graphs with excessive GPU RAM consumption. This is because, in our problem, maximum set size $|V| = 10000$ and each instance in our problem is a featurized tensor of dimension $10 \times 10000$. We believe that processing batches of such instances leads to the set-attention-blocks consuming huge memory. As mentioned by~\citet[Page\,16]{lee2019set}, the SAB block in \settx\  admits a maximum size of 2000 elements, in contrast, we have 10000 elements.
% \xhdr{DSF} The current number of parameters DSF involves a fixed concave function and thus, it demands fewer parameters. We attempted to train DSFs for various model sizes by varying the number of recursive layers (we attempted {2,5,10,30,40} layers). We observe that it works best for 2 layers (this is also the case in our method where $N=2$) which gives the number of parameters 21. With an increased number of layers (max layers= 40 corresponds to number of parameters = 401), we observe that the performance does not provide any improvement. This is because, when we apply many concave functions one after another, the gradient often quickly becomes very small for a slightly large size of $S$.

% \xhdr{SubMix} We use linear combinations of  four submodular functions. Hence, the number of parameters is 4.

\section{Additional details about experiments on subset selection for product recommendation}
\label{app:real}
\begin{table}[h]
\centering 
\maxsizebox{.6\hsize}{!}{
\begin{tabular}{|l||c|c|c|c|c|c|c|} \hline
Catgories & $|\Ucal|$ & $|V|$ & $\sum |S|$ & $\EE[|S|]$ & $\min_S |S|$ &$\max_S |S|$ \\ \hline \hline
 Gear & 4277 & 100 & 16288 & 3.808 & 3 & 10 \\ 

Bath & 3195 & 100 & 12147 & 3.802 & 3 & 11 \\ 

Health & 2995 & 62 & 11053 & 3.69 & 3 & 9 \\ 

Diaper & 6108 & 100 & 25333 & 4.148 & 3 & 15 \\ 

Toys & 2421 & 62 & 9924 & 4.099 & 3 & 14 \\ 

Bedding & 4524 & 100 & 17509 & 3.87 & 3 & 12 \\ 

Feeding & 8202 & 100 & 37901 & 4.621 & 3 & 23 \\ 

Apparel & 4675 & 100 & 21176 & 4.53 & 3 & 21 \\ 

Media & 1485 & 58 & 6723 & 4.527 & 3 & 19 \\ 

% Safety & 267 & 36 & 846 & 3.169 & 3 & 5 \\ 

% Furniture & 280 & 32 & 892 & 3.186 & 3 & 6 \\ 

% Strollers & 677 & 40 & 2290 & 3.383 & 3 & 7 \\ 

% Carseats & 483 & 34 & 1576 & 3.263 & 3 & 6 \\
\hline

\end{tabular} }
\caption{Amazon baby registry statistics.}
\label{tab:stats-data}
\end{table}

\subsection{Dataset description} 

As mentioned in Section~\ref{sec:real-main}, each dataset contains a universal set $V$ and a set of subsets $\Scal=\set{S}$. We summarize the details of the categories of the Amazon baby registry~\cite{gillenwater2014expectation} in Table~\ref{tab:stats-data}. For each categories, we first filter out those subsets $S$ for which $|S| \ge 3$.  Moreover, we use the $768$ dimensional BERT embedding of the description of each item $s\in V$ to compute $\zb_s$.
\subsection{Implementation details} 
We implemented \our, DSF and SubMix following the procedure
 described in the Appendix~\ref{app:syn}, except that we considered 
 we use 
$F_{\theta}(S) = \theta_1 \log(1+\sum_{s\in S}\theta_m ^\top \zb_s) + \theta_2 \log (1+\log(1+\sum_{s\in S} \theta_m ^\top \zb_s)) + \theta_3 \log(1+( \log (1+\log (1+\sum_{s\in S} \theta_m ^\top \zb_s))))$.
 for SubMix.
  Here, we train each trainable model for 30 epochs and choose the trained model that gives best mean Jaccard coefficient on the validation set in these 30 epochs.
 For maximizing DPP, Facility Location
and Disparity Min baselines, we used the library \url{https://github.com/decile-team/submodlib}.

\xhdr{Computing environment}
Our code was written in pytorch 1.7, running on a 16-core Intel(R) Xeon(R) Gold 6226R CPU@2.90GHz with 115 GB RAM, one Nvidia V100-32 GB GPU Card and Ubuntu 20.04 OS.

\subsection{License}
We collected Amazon baby registry dataset from \url{https://code.google.com/archive/p/em-for-dpps/} which comes under BSD license.
\newpage
\section{Additional experiments on real data}
\label{app:addl}

\subsection{Replication of  Table~\ref{tab:real} with standard deviation}
\label{app:app:subset}

% \xhdr{Efficiency comparison  with~\cite{tschiatschek2018differentiable}}
% The key motivation behind our proposed permutation adversarial data subset selection method is to ensure that the estimated parameters are invariant to the  order of the elements of the input subset.
% \citet{tschiatschek2018differentiable} achieve this goal by computing Monte Carlo average of the underlying likelihood over many samples. Here, we compare our method against this Monte Carlo approach. Figure~\ref{fig:efficiency} summarizes the results for Gear and Bath categories, which shows that our method is atleast $4$x faster than the proposal of~\cite{tschiatschek2018differentiable}.

  Here, we reproduce Table~\ref{tab:real} with standard deviation. Table~\ref{tab:real-app} shows the results.
\begin{table*}[h]
\centering
\adjustbox{max width=\hsize}{ \tabcolsep 3pt
\begingroup \footnotesize
\begin{tabular}{p{1.1cm}||c|ccccc}
\hline
& \multicolumn{6}{c}{\textbf{Mean Jaccard Coefficient (MJC)}}   \\ \hline
&\our  & DSF & SubMix & FL  & DPP & DisMin   \\ \hline\hline
Gear &  \textbf{ 0.101  $\pm$  0.003 } &  \sec{ 0.099  $\pm$  0.003 } &   0.028  $\pm$  0.002 &   0.019  $\pm$  0.001 &   0.014  $\pm$  0.001 &   0.013  $\pm$  0.001 \\ 

Bath &  \textbf{ 0.091  $\pm$  0.004 } &  \sec{ 0.087  $\pm$  0.003 } &   0.038  $\pm$  0.002 &   0.02  $\pm$  0.002 &   0.012  $\pm$  0.001 &   0.01  $\pm$  0.001 \\ 

Health &  \textbf{ 0.153  $\pm$  0.005 } &  \sec{ 0.142  $\pm$  0.004 } &   0.022  $\pm$  0.002 &   0.084  $\pm$  0.004 &   0.011  $\pm$  0.001 &   0.015  $\pm$  0.001 \\ 

Diaper &  \textbf{ 0.134  $\pm$  0.004 } &  \sec{ 0.115  $\pm$  0.004 } &   0.023  $\pm$  0.001 &   0.018  $\pm$  0.001 &   0.013  $\pm$  0.001 &   0.012  $\pm$  0.001 \\ 

Toys &  \textbf{ 0.157  $\pm$  0.006 } &  \sec{ 0.15  $\pm$  0.006 } &   0.025  $\pm$  0.002 &   0.064  $\pm$  0.003 &   0.029  $\pm$  0.002 &   0.029  $\pm$  0.002 \\ 

Bedding &  \textbf{ 0.203  $\pm$  0.005 } &  \sec{ 0.191  $\pm$  0.004 } &   0.028  $\pm$  0.002 &   0.015  $\pm$  0.001 &   0.043  $\pm$  0.002 &   0.047  $\pm$  0.002 \\ 

Feeding &  \textbf{ 0.1  $\pm$  0.002 } &  \sec{ 0.091  $\pm$  0.002 } &   0.026  $\pm$  0.001 &   0.023  $\pm$  0.001 &   0.02  $\pm$  0.001 &   0.019  $\pm$  0.001 \\ 

Apparel &  \textbf{ 0.101  $\pm$  0.003 } &  \sec{ 0.093  $\pm$  0.003 } &   0.036  $\pm$  0.002 &   0.022  $\pm$  0.001 &   0.016  $\pm$  0.001 &   0.016  $\pm$  0.001 \\ 

Media &  \textbf{ 0.135  $\pm$  0.006 } &  \sec{ 0.13  $\pm$  0.006 } &   0.029  $\pm$  0.003 &   0.035  $\pm$  0.003 &   0.029  $\pm$  0.002 &   0.025  $\pm$  0.002 \\ 

\hline\hline
& \multicolumn{6}{c}{\textbf{Mean Normalized Discounted Cumulative Gain@10 (Mean NDCG@10)}}   \\ \hline
&\our  & DSF & SubMix & FL  & DPP & DisMin   \\ \hline\hline

% Gear &  \textbf{ 0.475  $\pm$  0.003 } &  \textbf{ 0.475  $\pm$  0.003 } &   0.404  $\pm$  0.002 &   0.388  $\pm$  0.001 &   0.382  $\pm$  0.001 &   0.382  $\pm$  0.001 \\ 

% Bath &  \textbf{ 0.465  $\pm$  0.004 } &  \sec{ 0.447  $\pm$  0.003 } &   0.402  $\pm$  0.002 &   0.392  $\pm$  0.002 &   0.381  $\pm$  0.001 &   0.38  $\pm$  0.001 \\ 

% Health &  \textbf{ 0.531  $\pm$  0.005 } &  \sec{ 0.487  $\pm$  0.004 } &   0.397  $\pm$  0.002 &   0.486  $\pm$  0.004 &   0.381  $\pm$  0.001 &   0.383  $\pm$  0.001 \\ 

% Diaper &  \textbf{ 0.498  $\pm$  0.004 } &  \sec{ 0.481  $\pm$  0.003 } &   0.391  $\pm$  0.001 &   0.386  $\pm$  0.001 &   0.38  $\pm$  0.001 &   0.38  $\pm$  0.001 \\ 

% Toys &  \textbf{ 0.517  $\pm$  0.006 } &  \sec{ 0.509  $\pm$  0.005 } &   0.388  $\pm$  0.002 &   0.416  $\pm$  0.003 &   0.388  $\pm$  0.002 &   0.388  $\pm$  0.002 \\ 

% Bedding &  \textbf{ 0.583  $\pm$  0.005 } &  \sec{ 0.563  $\pm$  0.004 } &   0.39  $\pm$  0.001 &   0.384  $\pm$  0.001 &   0.398  $\pm$  0.001 &   0.4  $\pm$  0.001 \\ 

% Feeding &  \textbf{ 0.464  $\pm$  0.002 } &  \textbf{ 0.464  $\pm$  0.002 } &   0.389  $\pm$  0.001 &   0.382  $\pm$  0.001 &   0.38  $\pm$  0.001 &   0.379  $\pm$  0.001 \\ 

% Apparel &  \textbf{ 0.476  $\pm$  0.003 } &  \sec{ 0.469  $\pm$  0.003 } &   0.391  $\pm$  0.001 &   0.383  $\pm$  0.001 &   0.378  $\pm$  0.001 &   0.379  $\pm$  0.001 \\ 

% Media &  \textbf{ 0.488  $\pm$  0.006 } &  \sec{ 0.487  $\pm$  0.006 } &   0.395  $\pm$  0.004 &   0.396  $\pm$  0.003 &   0.385  $\pm$  0.002 &   0.383  $\pm$  0.002 \\ 
Gear &  \textbf{ 0.539  $\pm$  0.004 } &  \sec{ 0.538  $\pm$  0.004 } &   0.449  $\pm$  0.003 &   0.433  $\pm$  0.002 &   0.425  $\pm$  0.002 &   0.426  $\pm$  0.002 \\ 

Bath &  \textbf{ 0.52  $\pm$  0.004 } &  \sec{ 0.5  $\pm$  0.004 } &   0.447  $\pm$  0.002 &   0.433  $\pm$  0.002 &   0.427  $\pm$  0.002 &   0.422  $\pm$  0.002 \\ 

Health &  \textbf{ 0.597  $\pm$  0.005 } &  \sec{ 0.549  $\pm$  0.004 } &   0.449  $\pm$  0.003 &   0.54  $\pm$  0.005 &   0.425  $\pm$  0.002 &   0.435  $\pm$  0.002 \\ 

Diaper &  \textbf{ 0.562  $\pm$  0.004 } &  \sec{ 0.546  $\pm$  0.004 } &   0.447  $\pm$  0.002 &   0.44  $\pm$  0.002 &   0.435  $\pm$  0.002 &   0.435  $\pm$  0.002 \\ 

Toys &  \textbf{ 0.591  $\pm$  0.006 } &  \sec{ 0.577  $\pm$  0.006 } &   0.446  $\pm$  0.003 &   0.472  $\pm$  0.003 &   0.448  $\pm$  0.003 &   0.449  $\pm$  0.003 \\ 

Bedding &  \textbf{ 0.643  $\pm$  0.005 } &  \sec{ 0.623  $\pm$  0.004 } &   0.437  $\pm$  0.002 &   0.438  $\pm$  0.002 &   0.456  $\pm$  0.002 &   0.461  $\pm$  0.002 \\ 

Feeding &  \textbf{ 0.55  $\pm$  0.003 } &  \sec{ 0.547  $\pm$  0.003 } &   0.459  $\pm$  0.002 &   0.453  $\pm$  0.001 &   0.454  $\pm$  0.001 &   0.452  $\pm$  0.001 \\ 

Apparel &  \textbf{ 0.558  $\pm$  0.004 } &  \sec{ 0.55  $\pm$  0.004 } &   0.459  $\pm$  0.002 &   0.452  $\pm$  0.002 &   0.446  $\pm$  0.002 &   0.444  $\pm$  0.002 \\ 

Media &  \textbf{ 0.578  $\pm$  0.007 } &  \sec{ 0.578  $\pm$  0.006 } &   0.474  $\pm$  0.004 &   0.47  $\pm$  0.004 &   0.461  $\pm$  0.004 &   0.461  $\pm$  0.004 \\ 

\hline\hline
\end{tabular}
\endgroup}
\caption{Replica of Table~\ref{tab:real} with standard deviation. 
Here, we perform prediction of subsets in product recommendation task. Performance is measured in terms of Jaccard Coefficient (JC) and Normalized Discounted Cumulative Gain@10 (NDCG@10)
  for nine datasets the Amazon baby registry records, for \our, Deep submodular function (DSF), mixture of submodular functions (SubMix),
Facility location (FL), Determinantal point process (DPP) and Disparity Min (DisMin). In all experiments, we use training, test, validation folds of equal size.
Numbers in \textbf{bold}   (\sec{underline})
indicate best (second best) performer. }

 \label{tab:real-app}
\end{table*}

\subsection{Efficiency }
The key motivation of our proposed data subset selection method is to ensure that the estimated parameters are invariant to the  order of the elements of the input subset.
\citet{tschiatschek2018differentiable} achieve this goal by computing Monte Carlo average of the underlying likelihood over many samples.   In Figure~\ref{fig:efficiency}, we compare our method with their proposal,  which shows that our method is ${\ge}4\times$ faster.
 \begin{figure}[h]
\centering
%\vspace{-5mm}
\subfloat[Gear]{\includegraphics[width=0.3\hsize]{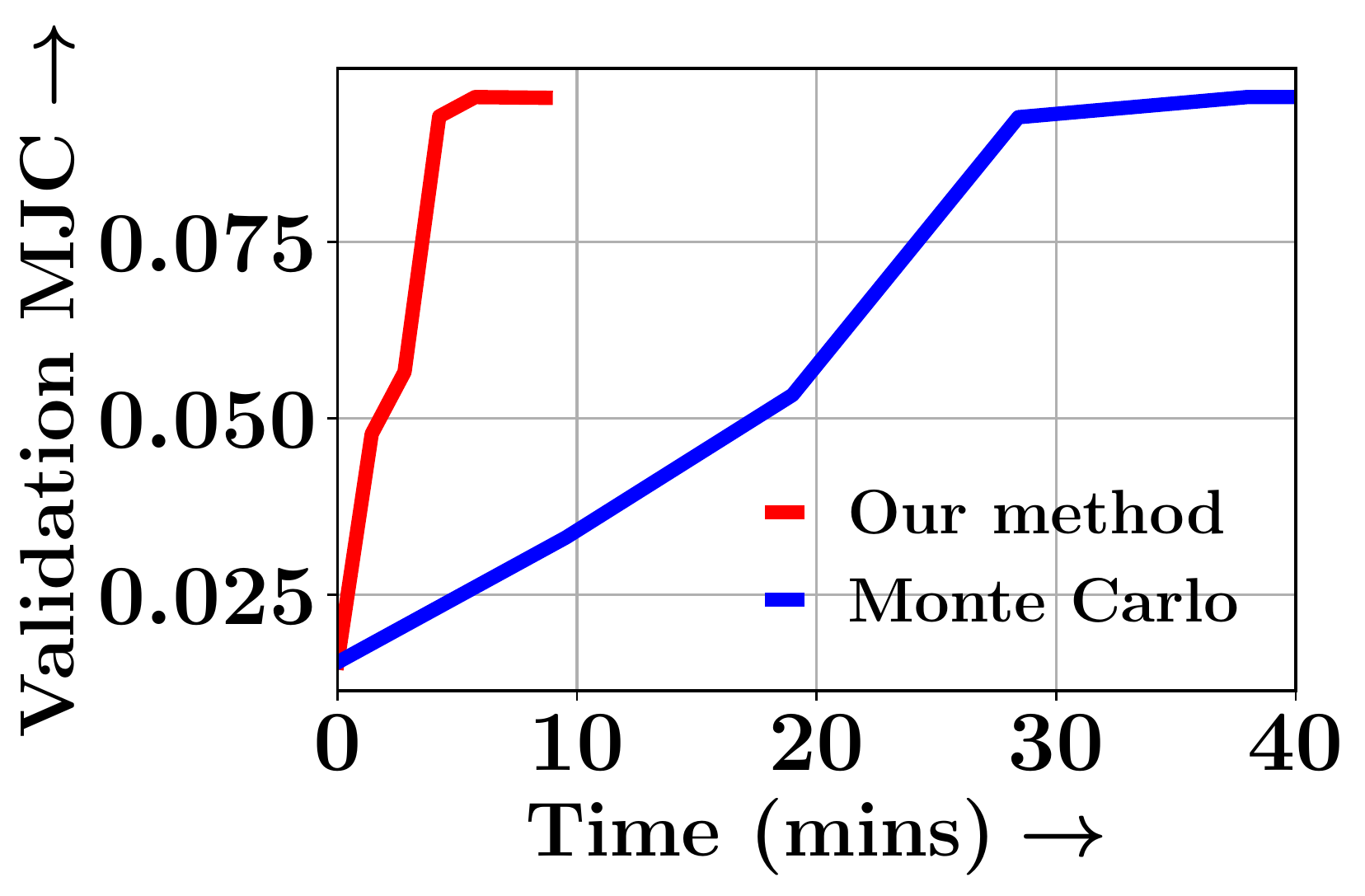}}\hspace{4mm}
\subfloat[Bath]{\includegraphics[width=0.3\hsize]{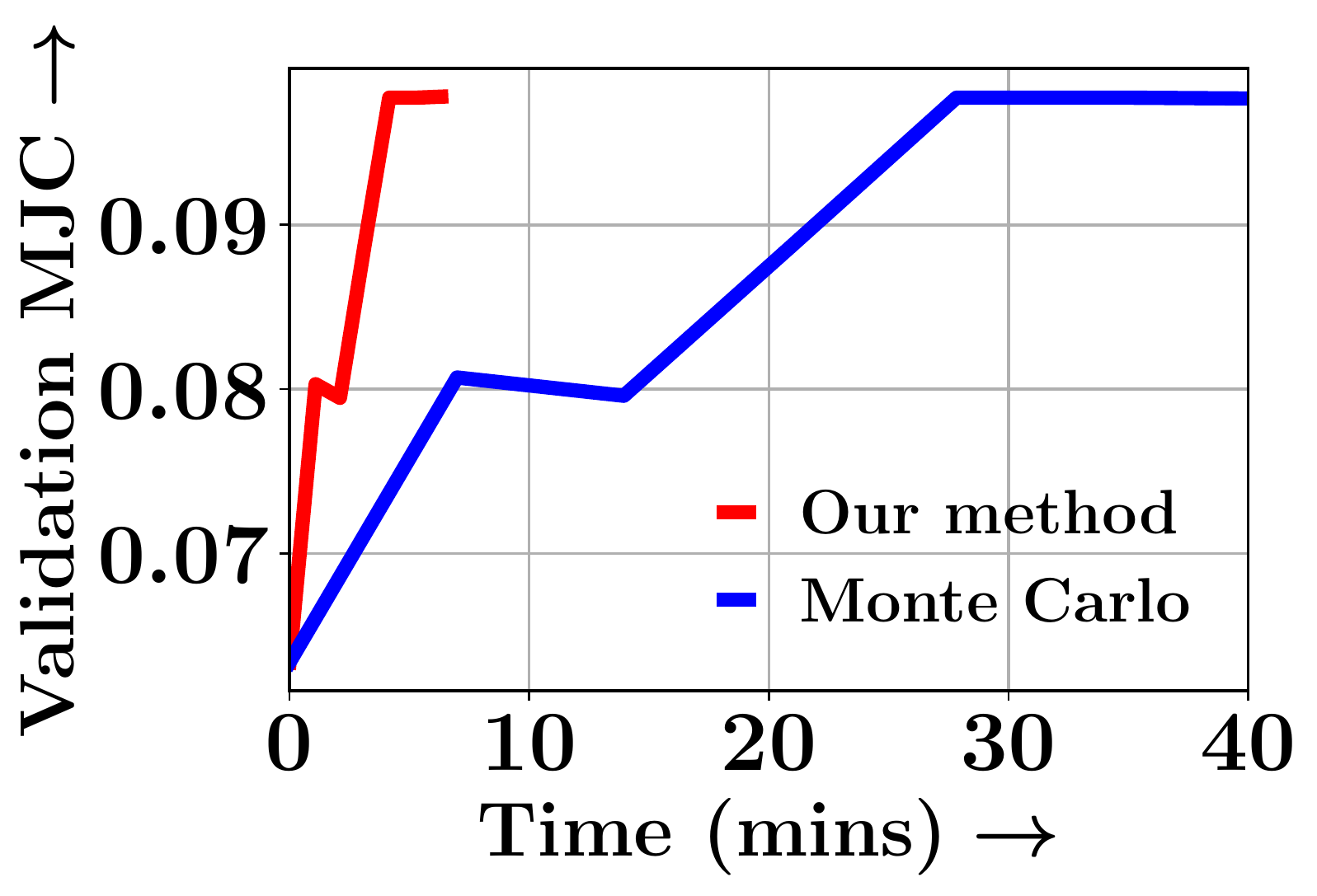}}
% \hspace{0.5cm}
% \subfloat[Bath, MJC]{ \includegraphics[width=0.20\hsize]{FIG/efficiencybathjc.pdf}}\hspace{2mm}
% \subfloat[Bath, MAP]{ \includegraphics[width=0.20\hsize]{FIG/efficiencybathmap.pdf}}\hspace{2mm}
% %\vspace{-2mm}
\caption{Variation of MJC as training progresses, for Gear and Bath categories.  Our proposed permutation adversarial subset selection is atleast $4\times$ faster than the Monte Carlo sampling method~\cite{tschiatschek2018differentiable}.}
% %\vspace{-1mm}
\label{fig:efficiency}
\end{figure}
\newpage

% \newpage
%\bibliography{refs}
%\bibliographystyle{plainnat}
% \bibliography{refs}
% \bibliographystyle{plainnat}
% \newpage
% %\bibliography{refs}
% %\bibliographystyle{plainnat}
% \bibliography{refs}
% \bibliographystyle{plainnat}
\end{document}